%% file: main.tex
\documentclass[final,12pt]{colt2023} 

\usepackage{hyperref}
\usepackage{url}
\usepackage{xspace}

\def\radar{\textsc{Radar}\xspace}
\def\radarv{\textsc{Radar}\textsubscript{v}\xspace}
\def\radars{\textsc{Radar}\textsubscript{s}\xspace}
\def\radarb{\textsc{Radar}\textsubscript{b}\xspace}

\SetKwComment{Comment}{/* }{ */}
\input{math_command.tex}

\title[Dynamic Regret on Geodesic Metric Spaces]{Minimizing Dynamic Regret on Geodesic Metric Spaces}
\usepackage{times}

\coltauthor{%
\Name{Zihao Hu} \Email{zihaohu@gatech.edu}\\
 \addr Georgia Institute of Technology
 \AND
 \Name{Guanghui Wang} \Email{gwang369@gatech.edu}\\
 \addr Georgia Institute of Technology
 \AND
 \Name{Jacob Abernethy} \Email{abernethyj@google.com}\\
 \addr Georgia Institute of Technology, Google Research
}

\begin{document}

\maketitle

\begin{abstract}%
  In this paper, we consider the sequential decision problem where the goal is to minimize the \emph{general dynamic regret} on a complete Riemannian manifold. The task of offline optimization on such a domain, also known as a \emph{geodesic metric space}, has recently received significant attention. The online setting has received significantly less attention, and it has remained an open question whether the body of results that hold in the Euclidean setting can be transplanted into the land of Riemannian manifolds where new challenges (e.g., \emph{curvature}) come into play.
     In this paper, we show how to get optimistic regret bound on manifolds with non-positive curvature whenever improper learning is allowed and propose an array of adaptive no-regret algorithms. To the best of our knowledge, this is the first work that considers general dynamic regret and develops ``optimistic'' online learning algorithms which can be employed on geodesic metric spaces.
\end{abstract}


\section{Introduction}
Online convex optimization (OCO) in Euclidean space is a well-developed area with numerous applications. In each round, the learner takes an action from a decision set, while the \textit{adversary}  chooses a  loss function. The long-term performance metric of the learner is (static) regret, which is defined as the difference between the learner's cumulative loss and the loss of playing the best-fixed decision in hindsight. As the name suggests, OCO requires both the losses and the decision set to be \emph{convex}. From the theoretical perspective, convex functions and sets are well-behaved objects with many desirable properties that are generally required to obtain tight regret bounds. 

Typical algorithms in OCO, such as \textit{mirror descent}, determine how one should adjust parameter estimates in response to arriving data, typically by shifting parameters against the gradient of the loss. But in many cases of interest, the underlying parameter space is not only non-convex but non-Euclidean. The \textit{hyperboloid}, for example, arising from the solution set of a degree-two polynomial, is a Riemannian manifold that has garnered interest as a tool in tree-embedding tasks \citep{lou2020differentiating}. On such manifolds, we do have a generalized notion of convexity, known as \textit{geodesic convexity} \citep{udriste2013convex}. There are many popular problems of interest \citep{hosseini2015matrix,vishnoi2018geodesic,sra2018geodesically} where the underlying objective function is geodesically convex (gsc-convex) under a suitable Riemannian metric. But there has thus far been significantly limited research on how to do \textit{adaptive learning} in such spaces and to understand when regret bounds are obtainable. 

\begin{table*}[!hbpt]\label{table-1}
    \centering
    \caption{Summary of bounds. $\delta$ describes  the discrepancy between the decision set and the comparator set. We define $\zeta$ in Def.~\ref{def1} and let $B_T\coloneqq\min\{V_T,F_T\}$.}
\begin{tabular}{cccc}
\toprule[1pt]
Algorithm& Type & Dynamic regret   \\
\hline 
\radar& Standard & $O(\sqrt{{\zeta}(1+P_T)T})$   \\
& Lower bound &$\Omega(\sqrt{(1+P_T)T})$  \\
\radarv& Gradient-variation & $O(\sqrt{{\zeta}(\frac{1+P_T}{\delta^2}
+V_T)(P_T+1)})$   \\
\radars& Small-loss & $O(\sqrt{{\zeta}((1+P_T){\zeta}+F_T)(P_T+1)})$    \\
\radarb& Best-of-both-worlds & {\small $O\lt(\sqrt{\zeta( P_T(\zeta+\frac{1}{\delta^2})+B_T+1)(P_T+1)+ B_T\ln T}\rt)$}    \\
\bottomrule[1pt]
\end{tabular}
\end{table*}

Let $\N$ be a gsc-convex subset of a geodesic metric space $\M$. In this paper, we consider the problem of minimizing the \emph{general dynamic regret} on $\N$, defined as
$$
  \text{D-Regret}_T := \sumT f_t(\x_t)-\sumT f_t(\u_t),
$$
where $\x_1, \ldots, \x_T \in \N$ is the sequence of actions taken by the learner, whose loss is evaluated relative to the sequence of ``comparator'' points $\u_1,\dots,\u_T\in\N$.  There has been recent work establishing that sublinear regret is possible as long as $\N$ and the $f_t$'s are gsc-convex, for example using a Riemannian variant of Online Gradient Descent ~\citep{wang2021no}. But so far there are no such results that elicit better D-Regret using \textit{adaptive} algorithms.

What do we mean by ``adaptive'' in this context? In the online learning literature there have emerged three key quantities of interest in the context of sequential decision making, the \textit{comparator path length}, the \textit{gradient variation}, and the \textit{comparator loss}, defined respectively as:
\begin{eqnarray}
     \textstyle P_T & := &\textstyle\sumsT d(\u_t,\u_{t-1}),\label{path} \\
    \textstyle  V_T & := & \textstyle\sum_{t=2}^T\sup_{\x\in\N}\|\nabla f_{t-1}(\x)-\nabla f_t(\x)\|^2, \nonumber \\
    \textstyle F_T & := & \textstyle\sumT f_t(\u_t). \nonumber
\end{eqnarray}
Let us start by considering regret minimization with respect to path length. While it has been observed that  $O(\sqrt{(1+P_T)T})$ is optimal in a minimax sense~\citep{zhang2018adaptive}, a great deal of research for the Euclidean setting~\citep{srebro2010smoothness,chiang2012online,rakhlin2013online} has shown that significantly smaller regret is achievable when any one of the above quantities is small. These are not just cosmetic improvements either, many fundamental applications of online learning rely on these adaptive methods and bounds. We give a thorough overview in Section~\ref{sec:related}.

The goal of the present paper is to translate to the Riemannian setting an array of adaptive regret algorithms and prove corresponding bounds. We propose a family of algorithms which we call \radar, for \underline{R}iemannian \underline{a}daptive \underline{d}yn\underline{a}mic \underline{r}egret. The three important variants of \radar are \radarv, \radars, and \radarb, we prove regret bounds on each, summarized in Table \ref{table-1}. We allow \emph{improper learning} for the gradient-variation bound, which means the player can choose $\x_1,\dots,\x_T$ from a slightly larger set $\N_{\delta G}$ (formally defined in Definition \ref{def2}).

As a general matter, convex constraints on a Riemannian manifold introduce new difficulties in optimization that are not typically present in the Euclidean case, and there has been limited work on addressing these.
To the best of our knowledge, there are only three papers 
considering how to incorporate constraints on manifolds, and these all make further assumptions on either the curvature or the diameter of the feasible set.  \cite{martinez2022global} only applies to hyperbolic and spherical spaces.  \cite{criscitiello2022negative} works for complete Riemannian manifolds with sectional curvature in $[-K,K]$ but the diameter of the decision set is at most $\textstyle O(\frac{1}{\sqrt{K}})$.
\cite{martinez2022accelerated} mainly works for locally symmetric Hadamard manifolds. Our paper is the first to consider the projective distortion in the \emph{online} setting that applies to all Hadamard manifolds as long as improper learning is allowed without imposing further constraints on the diameter or the curvature.

 Obtaining adaptive regret guarantees in the Riemannian setting is by no means a trivial task, as the new geometry introduces various additional technical challenges. Here is but one example: whereas the cost of a (Bregman) projection into a feasible region can be controlled using a generalized ``Pythagorean'' theorem in the Euclidean setting, this same issue becomes more difficult on a manifold as we encounter geometric distortion due to curvature. To better appreciate this, for a Hadamard manifold $\M$, assume the projection of $\x\in \M$ onto a convex subset $\N\subset\M$ is ${\z}$. While it is true that for any $\y\in\N\setminus \{\z\}$ the  angle between geodesics $\overline{\z\x}$ and $\overline{\z\y}$
 is obtuse, this property is only relevant at the tangent space of $\z$, yet we need to analyze gradients at the tangent space of ${\x}$. The use of \textit{parallel transport} between $\x$ and $\z$ unavoidably incurs extra distortion and could potentially lead to $O(T)$ regret.


The last challenge comes from averaging on manifolds. For example, many adaptive OCO algorithms rely on the \textit{meta-expert} framework, described by \cite{van2016metagrad}, that runs several learning algorithms in parallel and combines them through appropriately-weighted averaging. There is not, unfortunately, a single way to take convex combinations in a geodesic metric space, and all such averaging schemes need to account for the curvature of the manifold and incorporate the associated costs. We finally find the \emph{Fr\'echet mean} to be a desirable choice, but the analysis must proceed with care.

The key contributions of our work can be summarized as follows:
\begin{itemize}
    \item We develop the  optimistic mirror descent (OMD) algorithm on Hadamard manifolds\footnote{We focus on Hadamard manifolds in the main paper and extend the guarantee to CAT$(\kappa)$ spaces in Appendix \ref{cat-proof}.} in the online improper learning setting. Interestingly, we also show Optimistic Hedge, a variant of OMD, works for gsc-convex losses. We believe these tools may have significant applications to research in online learning and Riemannian optimization.
    \item We combine our analysis on OMD with the \emph{meta-expert framework} ~\citep{van2016metagrad} to get several adaptive regret bounds, as shown in Table \ref{table-1}.
    \item We  develop a novel dynamic regret lower bound, which renders our $O(\sqrt{\zeta(1+P_T)T})$ bound to be tight up to the geometric constant $\zeta$.
\end{itemize}




\section{Preliminaries}
In this section, we introduce background knowledge of OCO and Riemannian manifolds.

\paragraph{OCO in Euclidean space.} We first formally describe  OCO  in Euclidean space. For each round $t=1,\dots,T$, the learner makes a decision $\x_t\in\mathcal{X}$ based on historical losses $f_1,\dots,f_{t-1}$ where $\X$ is a convex decision set, and then the adversary reveals a convex loss function $f_t$. The goal of the learner is to minimize the difference between the cumulative loss and that of the best-fixed decision in hindsight: $
\text{Regret}_T=\sum_{t=1}^Tf_t(\x_t)-\min_{\x\in\mathcal{X}}\sum_{t=1}^Tf_t(\x),
$
 which is usually referred to as the \emph{static regret}, since the comparator is a fixed decision.


In the literature, there exist  a large number of  algorithms ~\citep{hazan2016introduction} on minimizing the static regret. However, the underlying assumption of the static regret is the adversary's behavior does not change drastically, which can be unrealistic in real applications. To resolve this issue,  dynamic regret stands out, which is defined as ~\citep{zinkevich2003online}
$$
\text{Regret}_T(\u_1,\dots,\u_T)=\sum_{t=1}^Tf_t(\x_t)-\sum_{t=1}^Tf_t(\u_t),
$$
where $\u_1,\dots,\u_T\in\X$ is a comparator sequence. Dynamic regret receives considerable attention recently ~\citep{besbes2015non,jadbabaie2015online,mokhtari2016online,NIPS:2017:Zhang,zhang2018adaptive,zhao2020dynamic,wan2021projection,L4DC:2021:Zhao,pmlr-v134-baby21a} due to its flexibility. However, dynamic regret  can be as large as $O(T)$ in general. Thus, regularizations need to be imposed on the comparator sequence to ensure no-regret online learning. A common assumption ~\citep{zinkevich2003online} is the path-length (see Equation \eqref{path}) of the comparator sequence to be bounded. We refer to the corresponding dynamic regret as \textit{general dynamic regret} because  any assignments to $\u_1,\dots,\u_T$ subject to the path-length constraint are feasible.



\textbf{Riemannian manifolds.} Here, we give a brief overview of Riemannian geometry, but this is a complex subject, and we refer the reader to, e.g., \cite{petersen2006riemannian} for a full treatment. A Riemannian manifold $(\mathcal{M},g)$ is a smooth manifold  $\M$ equipped with a Riemannian metric $g$. The tangent space $T_{\x}\M\cong\R^d$, generalizing the concept of the tangent plane, contains vectors tangent to any curve passing $\x$. The Riemannian metric $g$ induces the inner product $\langle\u,\v\rangle_{\x}$ and the Riemannian norm $\|\u\|_{\x}=\sqrt{\langle\u,\u\rangle_{\x}}$ where $\u,\v\in T_{\x}\M$ (we omit the reference point $\x$ when it is clear from the context). We use $d(\x,\y)$ to denote the Riemannian distance between $\x,\y\in\M$, which is the greatest lower bound of the length of all piecewise smooth curves joining $\x$ and $\y$.


A curve connecting $\x,\y\in\M$ is a geodesic if it is locally length-minimizing. For two points $\x,\y\in\M$, suppose there exists a geodesic $\gamma(t):[0,1]\rightarrow\M$ such that $\gamma(0)=\x,\gamma(1)=\y$ and $\gamma'(0)=\v\in T_{\x}\M$. The exponential map $\expmap_{\x}(\cdot):T_{\x}\M\rightarrow\M$ maps $\v\in T_{\x}\M$ to $\y\in \M$. Correspondingly, the inverse exponential map $\expmap_{\x}^{-1}(\cdot):\M\rightarrow T_{\x}\M$ maps $\y\in\M$ to $\v\in T_{\x}\M$. 
   Since traveling along a geodesic is of constant velocity, we indeed have $d(\x,\y)=\|\expmap_{\x}^{-1}\y\|_{\x}$. Also, it is useful to compare  tangent vectors in different tangent spaces. Parallel transport $\Gamma_{\x}^{\y}\u$ translates $\u$ from $T_{\x}\M$ to $T_{\y}\M$ and preserves the inner product, i.e., $\langle \u,\v\rangle_{\x} = \langle \Ga_{\x}^{\y}\u,\Ga_{\x}^{\y}\v\rangle_{\y}$.

The curvature of Riemannian manifolds reflects the extent to which the manifold differs from a Euclidean surface. For optimization purposes, it usually suffices to consider the sectional curvature.   Following \cite{zhang2016first,wang2021no}, in this paper we mainly consider Hadamard manifolds, which are complete and single-connected manifolds with non-positive sectional curvature. On such manifolds, every two points are connected by a unique and distance-minimizing geodesic by Hopf-Rinow Theorem ~\citep{petersen2006riemannian}.




A subset $\N$ of $\M$ is gsc-convex if for any $\x,\y\in\N$, there exists a geodesic connecting $\x,\y$ and fully lies in $\N$. A function $f:\N\rightarrow \mathbb{R}$ is gsc-convex if $\N$ is gsc-convex and the composition $f(\gamma(t))$ satisfies
$
f(\gamma(t))\leq (1-t)f(\gamma(0))+tf(\gamma(1))
$
for any geodesic $\gamma(t)\subseteq\N$ and $t\in[0,1]$. An alternative definition of geodesic convexity is
$$
\textstyle f(\y)\geq f(\x)+\langle\nabla f(\x),\expmap_{\x}^{-1}\y\rangle,\qquad\forall\;\x,\y\in\N.
$$
where the Riemannian gradient $\nabla f(\x)\in T_{\x}\M$ is the unique vector determined by
$
D f(\x)[\v] = \langle\v, \nabla f(\x)\rangle
$
and $D f(\x)[\v]$ is the differential of $f$ along $\v\in T_{\x}\M$.

Similarly, a $L$-geodesically-smooth (L-gsc-smooth) function $f$ satisfies 
$\|\Ga_{\y}^{\x}\nabla f(\y)-\nabla f(\x)\|\leq L\cdot d(\x,\y)$ for all  $\x,\y\in\N$, or
$$
\textstyle f(\y)\leq f(\x)+\langle\nabla f(\x),\expmap_{\x}^{-1}\y\rangle+\frac{L}{2}d(\x,\y)^2.
$$

\section{Related Work} \label{sec:related}
In this part, we briefly review past work on OCO in Euclidean space,  online optimization and  optimism on Riemannian manifolds.
\subsection{OCO in Euclidean Space}
\paragraph{Static regret.} We first consider work on static regret. In Euclidean space, it is well known that online gradient descent (OGD) guarantees $O(\sqrt{T})$ and $O(\log T)$ regret for convex and strongly convex losses \citep{hazan2016introduction}, which are also minimax optimal \citep{abernethy2008optimal}. However, the aforementioned bounds are not fully adaptive due to the dependence on $T$. Therefore, there is a tendency to replace $T$ with problem-dependent quantities. \cite{srebro2010smoothness} first notice that the smooth and non-negative losses satisfy the self-bounding property, thus establishing the small-loss bound $O(\sqrt{F_T^\star})$ where $F_T^{\star}=\sum_{t=1}^Tf_t(\x^\star)$ is the cumulative loss of the best action in hindsight.  \cite{chiang2012online} propose  extra-gradient
 to get $O(\sqrt{V_T})$  gradient-variation  regret bound for convex and smooth losses where $V_T=\sum_{t=2}^T\sup_{\x\in\X}\|\nabla f_{t-1}(\x)-\nabla f_t(\x)\|_2^2$. 
\cite{rakhlin2013online} generalize the work of \cite{chiang2012online} and propose optimistic mirror descent, which has become a standard tool in online learning since then.


    

\paragraph{Dynamic regret.} Now we switch to the related work on dynamic regret. \cite{zinkevich2003online} propose to use OGD to get a $O\lt(\eta T+\frac{1+P_T}{\eta}\rt)$ regret bound where $\eta$ is the step size, but the result turns out to be $O((1+P_T)\sqrt{T})$ since the  value of $P_T$ is unknown to the learner. The seminal work of \cite{zhang2018adaptive} use Hedge to combine the advice of experts with different step sizes, and show a  $O\lt(\sqrt{(1+P_T)T}\rt)$ regret. A matching lower bound is also established therein.  \cite{zhao2020dynamic} utilize smoothness to get a gradient-variation bound, a small-loss bound, and a best-of-both-worlds bound in Euclidean space.

\subsection{Online Learning and Optimism on Riemannian Manifolds}
In the online setting, \cite{becigneul2018riemannian} consider adaptive stochastic optimization on Riemannian manifolds but their results only apply to the Cartesian product of one-manifolds.
    \cite{maass2022tracking}
    study the \emph{restricted dynamic regret} on Hadamard manifolds under the gradient-free setting and provide $O(\sqrt{T}+P_T^\star)$ bound for gsc-strongly convex and gsc-smooth functions, where $P_T^\star$ is the path-length of the comparator formed by $\u_t=\argmin_{\x\in\X}f_t(\x)$. On Hadamard manifolds, \cite{wang2021no} apply Riemannian OGD (R-OGD) to get $O(\sqrt{T})$ upper bound and $\Omega(\sqrt{T})$ randomized lower bound. Comparatively, we focus on general and adaptive dynamic regret on Hadamard manifolds. Our minimax lower bound is also novel.

    There also exist algorithms considering optimism on Riemannian manifolds. \cite{zhang2022minimax} propose Riemannian Corrected Extra Gradient (RCEG) for unconstrained minimax optimization on manifolds. \cite{karimi2022riemannian} consider a Robbins-Monro framework on Hadamard manifolds which subsumes Riemannian stochastic extra-gradient. By imposing the weakly asymptotically coercivity and using a decaying step size, the trajectory is guaranteed to be finite \citep{karimi2022riemannian}. However, our paper is the first to consider the constrained case and the online setting. For the improper learning setting, we show a constant step size achieves the same guarantee as in Euclidean space.

\section{Path Length Dynamic Regret Bound  on Manifolds}
\label{radar}
In this section, we present the  results related to the minimax path-length bound on manifolds. Before diving into the details, following previous work \citep{zinkevich2003online,wang2021no}, we introduce some standard assumptions and definitions.
\begin{ass}\label{hada}
$\M$ is a Hadamard manifold and its sectional curvature is  lower bounded by $\kappa\leq 0$.
\end{ass}
\begin{ass}\label{diam}
The decision set $\N$ is a gsc-convex compact subset of $\M$ with diameter upper bounded by $D$, i.e., $\sup_{\x,\y\in\N} d(\x,\y)\leq D$. For optimistic online learning, we allow the player chooses decisions from  $\N_{\delta M}$, which is  defined in Definition \ref{def2} and the diameter becomes $(D+2\delta M)$.
\end{ass}
\begin{ass}\label{grad}
The norm of  Riemannian gradients are  bounded by $G$, i.e., $\sup_{\x\in\N}\|\nabla f_t(\x)\|\leq G.$ When improper learning is allowed, we assume $\sup_{\x\in\ndm}\|\nabla f_t(\x)\|\leq G.$
\end{ass}


\begin{defn}\label{def1}
Under Assumptions \ref{hada}, \ref{diam}, we denote
${\zeta}\coloneqq\sqrt{-\kappa} D \operatorname{coth}(\sqrt{-\kappa} D)$. When improper learning is allowed, ${\zeta}\coloneqq\sqrt{-\kappa} (D+2\delta M) \operatorname{coth}(\sqrt{-\kappa} (D+2\delta M))$, where $M$ is  in Definition \ref{def2}. {Note that, on manifolds of zero sectional curvature ($\kappa=0$), we have $\zeta=\lim_{x\to 0}x\cdot\coth{x}=1$.}


\end{defn} 

The seminal work of \cite{zinkevich2003online} shows that the classical OGD algorithm can minimize the general dynamic regret in  Euclidean space. Motivated by this, we  consider the Riemannian OGD (R-OGD) algorithm \citep{wang2021no}:
\begin{equation}
\label{alg:rogd}
    \x_{t+1}=\Pi_{\N}\expmap_{\x_t}(-\eta\nabla f_t(\x_t)),
\end{equation}
which is a natural extension of OGD to the manifold setting.  We show  R-OGD can also minimizes the general dynamic regret on manifolds. Due to page limitation, we postpone details to  Appendix \ref{app:radar}.
\begin{thm}\label{ROGD}
Suppose  Assumptions \ref{hada}, \ref{diam} and \ref{grad} hold. Then the general dynamical regret  of  R-OGD  defined in Equation \eqref{alg:rogd} satisfies
\begin{equation}
    \text{D-Regret}_T\leq \frac{D^2+2DP_T}{2\eta}+\frac{\eta{\zeta} G^2T}{2}.
\end{equation}
\end{thm}
Theorem \ref{ROGD} implies that R-OGD yields $O(\frac{P_T+1}{\eta}+\eta T)$ general dynamic regret bound, which means the optimal step size is $\eta= O{\scriptstyle\left(\sqrt{\frac{1+P_T}{T}}\right)}$.
 However, this configuration of $\eta$ is invalid, as $P_T$ is unknown to the learner. Although a sub-optimal choice for $\eta$, i.e.,  $\textstyle{\eta=O\left(\frac{1}{\sqrt{T}}\right)}$, is accessible,  the resulting algorithm  suffers  $O((1+P_T)\sqrt{T})$ regret. 
 
{The meta-expert framework \citep{van2016metagrad} consists of a meta algorithm and some expert algorithm instances. The constructions are modular such that we can use different meta algorithms and expert algorithms  to achieve different regret guarantees. For optimizing dynamic regret, the seminal work of \citet{zhang2018adaptive} propose Ader based on this framework.} In every round $t$, each expert runs OGD with a different step size, and the meta algorithm applies Hedge to learn the best weights. The step sizes used by the experts are carefully designed so that there always exists an expert which is almost optimal. The regret of Ader is $O(\sqrt{(1+P_T)T})$, which is minimax-optimal in Euclidean space \citep{zhang2018adaptive}.

However, it is unclear how to extend Ader to manifolds at first glance since we need to figure out the "correct" way to do averaging. In this paper, we successfully resolve this problem using the \emph{Fr\'echet mean} and the \emph{geodesic mean}. Our proposed algorithm, called \radar, consists of $N$ instances of the expert algorithm (Algorithm \ref{alg:expert}), each of which runs  R-OGD with a different step size, and a meta algorithm (Algorithm \ref{alg:meta}), which enjoys a regret approximately the same as the best expert. We denote the set of all step sizes $\{\eta_i\}$ by $\H$. In the $t$-th round, the expert algorithms submit all $\x_{t,i}$'s ($i=1,\dots,N$) to the meta algorithm. Then the meta algorithm either computes the Fr\'{e}chet mean or the geodesic mean (see Algorithm \ref{alg:gsc} in Appendix \ref{app:lem} for details) as $\x_t$. After receiving $f_t$,  the meta algorithm  updates the weight of each expert $w_{t+1,i}$ via Hedge and sends $\nabla f_t(\x_{t,i})$ to the $i$-th expert, which computes $\x_{t+1,i}$ by R-OGD. The regret of the meta algorithm of \radar can be bounded by Lemma \ref{meta-rader}.


\begin{minipage}{0.488\textwidth}
\begin{algorithm2e}[H]
\caption{\radar: Meta Algorithm}\label{alg:meta}
\KwData{Learning rate $\beta$, set of step sizes $\mathcal{H}$, initial weights $w_{1,i}=\frac{N+1}{i(i+1)N}$}
\For{$t=1,\dots,T$}{
    Receive  $\x_{t,i}$ from  experts with stepsize $\eta_i$\\
    $\x_t=\argmin_{\x\in\N}\sum_{i\in[N]}w_{t,i} d(\x,\x_{t,i})^2$\\
    Observe the loss function $f_t$\\
    Update $w_{t+1,i}$ by Hedge with $f_{t}(\mathbf{x}_{t,i})$\\
    Send gradient $\nabla f_t(\x_{t,i})$ to each expert
}
\end{algorithm2e}
\end{minipage}
\hfill
\begin{minipage}{0.466\textwidth}
    \begin{algorithm2e}[H]
\caption{\radar: Expert Algorithm}\label{alg:expert}
\KwData{ A step size $\eta_i$}
Let $\x_{1,i}^\eta$ be any point in $\N$\\
\For{$t=1,\dots,T$} {
    Submit $\x_{t,i}$ to the meta algorithm\\
     Receive gradient $\nabla f_t(\x_{t,i})$ from the meta algorithm\\
     Update:\\
     $\x_{t+1,i}=\Pi_{\N}\expmap_{\x_{t,i}}(-\eta_i\nabla f_t(\x_{t,i}))$\\
    }
\end{algorithm2e}
\end{minipage}



\begin{lemma}\label{meta-rader}
Under Assumptions \ref{hada}, \ref{diam}, \ref{grad}, and setting $\beta=\sqrt{\frac{8}{G^2D^2T}}$,  the regret of Algorithm \ref{alg:meta} satisfies
$$
\textstyle\sum_{t=1}^Tf_t(\x_t)-\sum_{t=1}^Tf_t(\x_{t,i})\leq\sqrt{\frac{G^2D^2T}{8}}\left(1+\ln\frac{1}{w_{1,i}}\right).
$$
\end{lemma}

We show that, by configuring the step sizes in $\H$ carefully, \radar  ensures a $O(\sqrt{(1+P_T)T})$ bound on geodesic metric spaces.

\begin{thm}\label{RAder}
Set $\textstyle\mathcal{H}=\left\{\eta_i=2^{i-1}\sqrt{\frac{D^2}{G^2{\zeta} T}}\big| i\in[N]\right\}$
where $N=\lceil\frac{1}{2}\log_2(1+2T)\rceil+1$ and $\beta=\sqrt{\frac{8}{G^2D^2T}}$. Under Assumptions \ref{hada}, \ref{diam}, \ref{grad},  for any comparator sequence $\u_1,\dots,\u_T\in\N$, the general dynamic regret of \radar satisfies
\begin{equation*}
    \text{D-Regret}_T=O(\sqrt{{\zeta}(1+P_T)T}).
\end{equation*}
\end{thm}

\begin{remark}
Note that if  $\M$ is Euclidean space, then ${\zeta}=1$ and we get $O(\sqrt{(1+P_T)T)}$ regret, which is the same as in \cite{zhang2018adaptive}.

\end{remark}

{A disadvantage of \radar is $\Theta(\log T)$ gradient queries are required at each round. In Euclidean space, \citet{zhang2018adaptive} use a linear surrogate loss to achieve the same bound by $O(1)$ gradient queries. But on manifolds, the existence of such functions implies the sectional curvature of the manifold is everywhere $0$~\citep{kristaly2016convexities}. It is interesting to investigate if $\Omega(\log T)$ gradient queries are necessary to achieve dynamic regret on manifolds. We would also like to point out that $O(\log T)$ is reasonable small and the work of \citet{zhang2018adaptive} still needs $O(\log T)$ computational complexity per round.} 

Using the Busemann function as a bridge, we  show the following dynamic regret lower bound, with proof deferred to Appendix \ref{app:lb}.
\begin{thm}\label{lb_dynamic}
    There exists a comparator sequence which satisfies $\sumsT d(\u_t,\u_{t-1})\leq P_T$ and encounters $\Omega(\sqrt{(1+P_T)T})$
       dynamic regret on Hadamard manifolds.
\end{thm}

Although the regret guarantee in Theorem \ref{RAder}
is optimal up to constants in terms of $T$ and $P_T$ by considering the corresponding lower bound, it still depends on $T$ and thus cannot adapt to mild environments. In Euclidean space,  the smoothness of losses induces adaptive regret bounds, including  the gradient-variation bound \citep{chiang2012online} and the small-loss bound \citep{srebro2010smoothness}. It is then natural to ask if similar bounds can be established on manifolds by assuming gsc-smoothness. We provide an affirmative answer to this question and show how to get problem-dependent bounds under the \radar framework.


\section{Gradient-variation Bound on Manifolds}
\label{radarv}
In this section, we show how to obtain the gradient-variation bound on manifolds under the \radar framework with alternative expert and meta algorithms.

\paragraph{Expert Algorithm.} For minimax optimization on Riemannian manifolds, \cite{zhang2022minimax} propose  Riemannian Corrected Extra Gradient (RCEG),   which performs the following iterates:
\begin{equation*}
    \begin{split}
        \textstyle&\x_{t}=\expmap_{\y_t}(-\eta \nabla f_{t-1}(\y_t))\\
        &\y_{t+1}=\expmap_{\x_t}\left(-\eta\nabla f_t(\x_t)+\expmap_{\x_t}^{-1}(\y_t)\right).
    \end{split}
\end{equation*}
However, this algorithm does not work in the constrained case, which has been left as an open problem \citep{zhang2022minimax}. The online improper learning setting~\citep{hazan2018online,pmlr-v134-baby21a} allows the decision set to be different from (usually larger than) the set of strategies we want to compete against. Under such a setting, we find the geometric distortion due to projection can be bounded in an elegant way, and
generalize RCEG to incorporate an optimism term $M_t\in T_{\y_t}\M$. 

\begin{defn}\label{def2}
We use $M_t$ to denote the optimism at round $t$ and assume there exists $M$ such that $\|M_t\|\leq M$ for all t.     We define $\N_c= \{\x|d(x,\N)\leq c\}$ where $d(\x,\N)\coloneqq\inf_{\y\in\N}d(\x,\y)$. In the improper setting, we allow the player  to choose decisions from $\ndm$. 
\end{defn}

\begin{thm}\label{var-expert}
Suppose all losses $f_t$ are  $L$-gsc-smooth on $\M$. Under Assumptions \ref{hada}, \ref{diam}, \ref{grad}, the iterates  
\begin{equation}\label{omd}
    \begin{split}
        \textstyle &\x_{t}'=\expmap_{\y_t}(-\eta M_t)\\
        &\x_t=\Pi_{\ndm}\x_t'\\
        &\y_{t+1}=\Pi_{\N}\expmap_{\x_t'}\left(-\eta\nabla f_t(\x_t')+\expmap_{\x_t'}^{-1}(\y_t)\right).
    \end{split}
\end{equation}
 satisfies
 \begin{equation*}
    \begin{split}
        \sum_{t=1}^Tf_t(\x_t)-\sum_{t=1}^Tf_t(\u_t)\leq\eta{\zeta}\sumT\|\nabla f_t(\y_t)-M_t\|^2+\frac{D^2+2DP_T}{2\eta}.
    \end{split}
\end{equation*}
for any  $\u_1,\dots,\u_T\in\N$ and $\eta\leq\frac{\delta M}{G+(G^2+2\zeta\delta^2M^2L^2)^{\frac{1}{2}}}$. Specifically, we achieve $\eta{\zeta}V_T+\frac{D^2+2DP_T}{2\eta}$ regret by  choosing $M_t=\nabla f_{t-1}(\y_t)$. In this case, $M=G$ and we need $\eta\leq\frac{\delta}{1+(1+2\zeta\delta^2L^2)^{\frac{1}{2}}}$.

\end{thm}
\paragraph{Proof sketch.} We use the special case $M_t=\nabla f_{t-1}(\y_t)$ to illustrate the main idea of the proof. We first decompose $f_t(\x_t)-f_t(\u_t)$ into two terms,
\[
    \begin{split}       
    & f_t(\x_t)-f_t(\u_t)={(f_t(\x_t)-f_t(\x_t'))}+{(f_t(\x_t')-f_t(\u_t))}\\
    \leq&G\cdot d(\x_t,\x_t')+{(f_t(\x_t')-f_t(\u_t))}\leq \underbrace{G\cdot d(\x_t',\y_t)}_{\text{troublesome term } 1}+\underbrace{(f_t(\x_t')-f_t(\u_t))}_{\text{unconstrained RCEG}}
    \end{split}
    \]
    where the first inequality is because the gradient Lipschitzness condition, and the second one follows from the non-expansiveness of the projection.
    For the unconstrained RCEG term, we have the following decomposition,
    \[
    \begin{split}
        &f_t(\x_t')-f_t(\u_t)\leq \underbrace{\frac{1}{2\eta}(2\eta^2\zeta L^2-1)d(\x_t',\y_t)^2}_{\text{troublesome term }2}\\
        +&\underbrace{\eta\zeta\|\nabla f_t(\y_t)-\nabla f_{t-1}(\y_t)\|^2}_{\eta\zeta V_T}+\underbrace{\frac{1}{2\eta}\lt(d(\y_t,\u_t)^2-d(\y_{t+1},\u_t)^2\rt)}_{\frac{D^2+2DP_T}{2\eta}}
    \end{split}
    \]
    where the second and the third term corresponds to the gradient variation term and the dynamic regret term, respectively.
    
   In the improper learning setting, we can show $d(\x_t',\y_t)\geq\delta G$.  Combining both troublesome terms, it suffices to find $\eta$ which satisfies
    \[
    2\eta G+2\eta^2\zeta L^2\lambda-\lambda\leq 0,\;\forall \lambda\coloneqq d(\x_t',\y_t)\geq\delta G.
    \]

\begin{remark}
{We generalize Theorem \ref{var-expert} from Hadamard manifolds to CAT$(\kappa)$ spaces in Appendix \ref{cat-proof}.} Note that although we allow the player to make improper decisions, $V_T$ is still defined on $\N$ instead of $\N_{\delta G}$. For the static setting, $P_T=0$ and the resulting regret bound is $O(\sqrt{V_T}+\frac{1}{\delta})$. Also, in this setting, we can use an adaptive step-size
\[
\eta_t=\min\lt\{\frac{1}{\sqrt{1+\sum_{s=2}^t\|\nabla f_t(\y_t)-\nabla f_{t-1}(\y_t)\|^2}},\frac{\delta}{1+(1+2\zeta\delta^2L^2)^{\frac{1}{2}}}\rt\}
\]
to eliminate the dependence on $V_T$.
\end{remark}
\paragraph{Meta algorithm.} Intuitively, we can run OMD with different step sizes and
 apply a meta algorithm to estimate the optimal step size. Previous studies in learning with multiple step sizes usually adopt Hedge to aggregate the experts' advice. However, the regret of Hedge is $O(\sqrt{T\ln N})$ and thus is undesirable for our purpose. Inspired by optimistic online learning \citep{rakhlin2013online,syrgkanis2015fast}, \cite{zhao2020dynamic} adopt Optimistic Hedge as the meta algorithm to get $O(\sqrt{(V_T+P_T)P_T})$ gradient-variation bound. After careful analysis, we show Optimistic Hedge works for gsc-convex losses regardless of the geometric distortion and  get the desired gradient-variation bound.

\begin{algorithm2e}[H]
\caption{\radarv: Expert Algorithm}\label{alg:expert_v}
\KwData{A step size $\eta_i$}
Let $\x_{1,i}^\eta$ be any point in $\N$\\
\For{$t=1,\dots,T$} {
     Submit $\x_{t,i}$ to the meta algorithm\\
     Receive gradient $\nabla f_t(\cdot)$ from the meta algorithm\\
     Each expert runs Equation \eqref{omd} with $M_t=\nabla f_{t-1}(\y_t)$, $M=G$ and step size $\eta_i$\\
    }
\end{algorithm2e}

\begin{algorithm2e}[H]
\caption{\radarv: Meta Algorithm}\label{alg:meta_v}
\KwData{A learning rate $\beta$, a set of step sizes $\H$, initial weights $w_{1,i}=w_{0,i}=\frac{1}{N}$}
\For{$t=1,\dots,T$} {
     Receive all $\x_{t,i}$'s from  experts with step size $\eta_i$\\
       $\bx_t=\argmin_{\x\in\N_{\Gd}}\sum_{i\in[N]}w_{t-1,i} d(\x,\x_{t,i})^2$\\
     Update $w_{t,i}\propto\exp\left(-\beta\left(\sum_{s=1}^{t-1}\ell_{s,i}+m_{t,i}\right)\right)$ by Equation \eqref{radarv_opt}\\
      $\x_t=\argmin_{\x\in\N_{\Gd}}\sum_{i\in[N]}w_{t,i} d(\x,\x_{t,i})^2$\\
     Observe $f_t(\cdot)$ and send $\nabla f_t(\cdot)$ to experts\\
}
\end{algorithm2e}

We denote $\l_t,\m_t\in\mathbb{R}^N$ as the surrogate loss and the optimism at round $t$. The update rule of Optimistic Hedge is: 
\[
\textstyle w_{t,i}\propto\exp\left(-\beta\left(\sum_{s=1}^{t-1}\ell_{s,i}+m_{t,i}\right)\right),
\]
 which achieves adaptive regret due to the optimism. The following technical lemma \citep{syrgkanis2015fast} is critical for our analysis of Optimistic Hedge, and the proof is in Appendix \ref{app:lem:opthedge} for completeness.

\begin{lemma}\label{optHedge}
For any $i\in[N]$,  Optimistic Hedge satisfies
\begin{equation*}
\begin{split}
   \textstyle \sum_{t=1}^T\la\w_t,\l_t\ra-\ell_{t,i}\leq\frac{2+\ln N}{\beta}+\beta\sum_{t=1}^T\|\l_t
    -\m_t\|_{\infty}^2-\frac{1}{4\beta}\sum_{t=2}^T\|\w_t-\w_{t-1}\|_1^{2}
\end{split}
\end{equation*}
\end{lemma}

Following the insightful work of \cite{zhao2020dynamic}, we also adopt the Optimistic Hedge algorithm as the meta algorithm, but there are some key differences in the design of the surrogate loss and optimism. To respect the Riemannian metric, we  propose the following:
\begin{equation}\label{radarv_opt}
    \begin{split}
        &\ell_{t,i}=\la\nabla f_t(\x_t),\expmap_{\x_t}^{-1}\x_{t,i}\ra \\
        &m_{t,i}=\la\nabla f_{t-1}(\bar{\x}_t),\expmap_{\bx_t}^{-1}\x_{t,i}\ra \\
    \end{split}
\end{equation}
where $\x_t$ and $\bx_t$ are Fr\'echet averages of $\x_{t,i}$ w.r.t. linear combination coefficients $\w_{t}$ and $\w_{t-1}$ respectively. Under the Fr\'echet mean, we can show
\[
f_t(\x_t)-f_t(\x_{t,i})\leq\la\w_t,\l_t\ra-\ell_{t,i},
\]
which ensures Lemma \ref{optHedge} can be applied to bound the meta-regret and the geodesic mean does not meet this requirement. We also emphasize that the design of the surrogate loss and  optimism is highly non-trivial. As we will see in the proof of Theorem \ref{opthedge-reg}, the combination of the surrogate loss and the gradient-vanishing property of the Fr\'{e}chet mean ensures Lemma \ref{optHedge} can be invoked to upper bound the regret of the meta algorithm. However,  $\m_{t}$ cannot rely on $\x_t$ thus, we need to design  optimism based on the tangent space of ${\bx_t}$, which  incurs extra cost. Luckily, under Equation \eqref{radarv_opt}, we find a reasonable upper bound of this geometric distortion by showing 

\[
    \begin{split}       
    \|\l_t-\m_t\|_{\infty}^2&\leq O(1)\cdot\sup_{\x\in\N_{\delta G}}\|\nabla f_t(\x)-\nabla f_{t-1}(\x)\|^2+O(1)\cdot d(\x_t,\bx_t)^2\\
    &\leq O(1)\cdot\sup_{\x\in\N_{\delta G}}\|\nabla f_t(\x)-\nabla f_{t-1}(\x)\|^2+\Tilde{O}(1) \cdot \|\w_t-\w_{t-1}\|_1^2.
    \end{split}
    \]
    Thus we can apply the negative term in Lemma \ref{optHedge} to eliminate undesired terms in $\|\l_t-\m_t\|_{\infty}^2$.

 Algorithms \ref{alg:expert_v} and \ref{alg:meta_v} describe the expert algorithm and meta algorithm of \radarv. We show 
the meta-regret and total regret of \radarv in Theorems \ref{opthedge-reg} and \ref{var-reg}, respectively. Detailed proof in this section is deferred to Appendix \ref{app:radarv}.
\begin{thm}\label{opthedge-reg}
Assume all losses are $L$-gsc-smooth on $\M$. Then under Assumptions \ref{hada}, \ref{diam}, \ref{grad},   the regret of Algorithm \ref{alg:meta_v} satisfies:
\begin{equation*}
\begin{split}
   \textstyle & \sum_{t=1}^Tf_t(\x_t)-\sum_{t=1}^Tf_t(\x_{t,i})\\
   \leq&\frac{2+\ln N}{\beta}+3D^2\beta (V_T+G^2)+\sumsT\left(3\beta (D^4L^2+D^2G^2{\zeta}^2)-\frac{1}{4\beta}\right)\|\w_t-\w_{t-1}\|_1^2.
\end{split}
\end{equation*}
\end{thm}

\begin{thm}\label{var-reg}
Let $\textstyle\beta=\min\left\{\sqrt{\frac{2+\ln N}{3D^2V_T}},\frac{1}{\sqrt{12(D^4L^2+D^2G^2{\zeta}^2)}}\right\}$, $\textstyle\H=\left\{\eta_i=2^{i-1}\sqrt{\frac{D^2}{8{\zeta} G^2T}}\big|i\in[N]\right\}$, where $\textstyle N=\lt\lceil\frac{1}{2}\log\frac{8\zeta\delta^2G^2T}{({1+(1+2\zeta\delta^2L^2)^{\frac{1}{2}}})^2}\rt\rceil+1$. Assume all losses are $L$-gsc-smooth on $\M$ and allow improper learning. Under Assumptions \ref{hada}, \ref{diam} and \ref{grad},
 the  regret of \radarv satisfies
$$
\text{D-Regret}_T={O}\left(\sqrt{{\zeta}(V_T+(1+P_T)/\delta^2)(1+P_T)}\right).
$$
\end{thm}
In Theorem \ref{var-reg}, $\beta$ relies on $V_T$, and this dependence can be eliminated by showing a variant of Lemma \ref{optHedge} with an adaptive learning rate $\beta_t$.


\section{Small-loss Bound on Manifolds}
\label{radars}
For dynamic regret, the small-loss bound  replaces the dependence on $T$ by $F_T=\sum_{t=1}^Tf_t(\u_t)$, which adapts to the function values of the comparator sequence. 
In Euclidean space,  \cite{srebro2010smoothness} show this  adaptive regret by combining OGD with the self-bounding property of smooth and non-negative functions, which reads
$
\|\nabla f(\x)\|_2^2\leq 4L\cdot f(\x)
$
where $L$ is the smoothness constant. We show a similar conclusion on manifolds and defer proof details in this part to Appendix \ref{app:radars}. 
\begin{lemma}\label{self-bound}
Suppose $f:\M\rightarrow\R$ is both $L$-gsc-smooth and non-negative on its domain where $\M$  is a Hadamard manifold, then we have
$\|\nabla f(\x)\|^2\leq 2L\cdot f(\x)$.

\end{lemma}

To facilitate the discussion, we denote $\bar{F}_T=\sumT f_t(\x_t)$ and $\bar{F}_{T,i}=\sumT f_t(\x_{t,i})$. 
We use R-OGD as the expert algorithm (Algorithm \ref{alg:expert}) and Hedge with surrogate loss
$$
\ell_{t,i}=\la\nabla f_t(\x_t),\invexp_{\x_t}\x_{t,i}\ra
$$
as the meta algorithm (Algorithm \ref{alg:meta}). The following Lemma
considers the regret of a single expert and shows that R-OGD achieves a small-loss dynamic regret on geodesic metric spaces.
\begin{lemma}\label{sml-expert}
Suppose all losses are $L$-gsc-smooth and non-negative on $\M$. Under Assumptions \ref{hada}, \ref{diam},    by choosing any step size $\eta\leq \frac{1}{2{\zeta}L}$, R-OGD achieves $O\left(\frac{P_T}{\eta}+\eta F_T\right)$ regret.
\end{lemma}

Again, we can not directly set $\textstyle\eta=O\lt(\frac{1+P_T}{F_T}\rt)$ because $P_T$ is unknown, which is precisely why we need the meta algorithm. The meta-regret of Hedge is as follows.
\begin{lemma}\label{sml-meta}
Suppose all losses are $L$-gsc-smooth and non-negative on $\M$.
Under Assumptions \ref{hada}, \ref{diam}, by setting the learning rate of Hedge as $\beta=\sqrt{\frac{(2+\ln N)}{D^2\bar{F}_T}}$, the regret of the meta algorithm satisfies
\begin{equation*}
    \begin{split}
        \textstyle\sum_{t=1}^Tf_t(\x_t)-\sum_{t=1}^Tf_t(\x_{t,i})\leq 8D^2L(2+\ln N)
        +\sqrt{8D^2L(2+\ln N)F_{T,i}}.
    \end{split}
\end{equation*}
\end{lemma}
Now as we have the guarantee for both the expert algorithm and the meta algorithm, a direct combination yields the following general dynamic regret guarantee.
\begin{thm}\label{sml-reg}
Suppose all losses are $L$-gsc-smooth and non-negative on $\M$. Under Assumptions \ref{hada}, \ref{diam}. Setting $\H=\lt\{\eta_i=2^{i-1}\sqrt{\frac{D}{2 {\zeta} LGT}}\big|i\in[N]\rt\}$ where $N=\lc\frac{1}{2}\log\frac{GT}{2 LD{\zeta}}\rc+1$ and {\small $\beta=\sqrt{\frac{(2+\ln N)}{D^2\bar{F}_T}}$}. Then for any comparator $\u_1,\dots,\u_T\in\N$, we have


$$
\text{D-Regret}_T=O(\sqrt{\zeta({\zeta}(P_T+1)+F_T)(P_T+1)}).
$$

\end{thm}

\begin{remark}
If we take $\M$ as Euclidean space, the regret guarantee shown in Theorem \ref{sml-reg} becomes { $O(\sqrt{(P_T+F_T+1)(P_T+1)})$}, which matches
 the result of \cite{zhao2020dynamic}.
\end{remark}
\section{Best-of-both-worlds Bound on Manifolds}
\label{radarb}
Now  we already achieve the gradient-variation bound and the small-loss bound on manifolds. To highlight the differences between them, we provide an example in Appendix \ref{bbw} to show under certain scenarios, one bound can be much tighter than the other.
The next natural question is, is that possible to get a best-of-both-worlds bound on manifolds?

We  initialize $N\coloneqq N^v+N^s$ experts as shown in Theorems \ref{var-reg} and \ref{sml-reg} where $N^v$ and $N^s$ are the numbers of experts running OMD and R-OGD, respectively. For each expert $i\in [N]$, the surrogate loss and the optimism are
\begin{equation}\label{radarb_opt}
    \begin{split}
        &\ell_{t,i}=\la\nabla f_t(\x_t),\expmap_{\x_t}^{-1}\x_{t,i}\ra \\
        &m_{t,i}=\gamma_t\la\nabla f_{t-1}(\bar{\x}_t),\expmap_{\bx_t}^{-1}\x_{t,i}\ra. \\
    \end{split}
\end{equation}
 $\gamma_t$ controls the optimism used in the meta algorithm. When $\gamma_t=1$, the optimism for the gradient-variation bound is recovered, and $\gamma_t=0$ corresponds to the optimism for the small-loss bound.

Following \cite{zhao2020dynamic}, we use Hedge for two experts to get a best-of-both-worlds bound. The analysis therein relies on the strong convexity of $\|\nabla f_t(\x_t)-\m\|_2^2$ in $\m$, which is generally not the case on manifolds. So an alternative scheme needs to be proposed. We denote
\begin{equation}
    \begin{split}
        &m_{t,i}^v=\la\nabla f_{t-1}(\bx_t),\invexp_{\bx_t}\x_{t,i}\ra\\
        &m_{t,i}^s=0,\\
    \end{split}
\end{equation}
 while $\m_t^v$ and $\m_t^s$ be the corresponding vectors respectively. Then  $\m_t=\gamma_t\m_t^v+(1-\gamma_t)\m_t^s$, which is exactly the combination rule of Hedge. The function $\|\l_t-\m\|_{\infty}^2$ is convex with respect to $\m$ but not strongly convex so we instead use  $d_t(\m)\coloneqq \|\l_t-\m\|_{2}^2$ for Hedge, and the learning rate is updated as
 {
 \begin{equation}\label{gamma}
     {\gamma_t=\frac{\exp\lt(-\tau\lt(\sum_{r=1}^{t-1}d_r(\m_r^v)\rt)\rt)}{\exp\lt(-\tau\lt(\sum_{r=1}^{t-1}d_r(\m_r^v)\rt)\rt)+\exp\lt(-\tau\lt(\sum_{r=1}^{t-1}d_r(\m_r^s)\rt)\rt)} }
\end{equation}
 }

Algorithm \ref{alg:meta_b}  summarizes the meta algorithm as well as the expert algorithm for \radarb.

 \begin{algorithm2e}
\caption{\radarb: Algorithm}\label{alg:meta_b}
\KwData{ Learning rates $\beta$ for Optimistic Hedge and $\gamma_t$ for Hedging the two experts, $\mathcal{H}=\{\eta_i\}$ consists of $N=N^v+N^s$ step sizes, $\tau=\frac{1}{8NG^2D^2}$}
\For{$t=1,\dots,T$} {
     Run Algorithms \ref{alg:expert_v} and \ref{alg:expert} on the first $N^v$ experts and the later $N^s$ experts, resp.
      $\bx_t=\argmin_{\x\in\N_{\Gd}}\sum_{i\in[N]}w_{t-1,i} d(\x,\x_{t,i})^2$\\
     Update $\gamma_t$ as in Equation \eqref{gamma}\\
     Update $w_{t,i}\propto\exp\left(-\beta\left(\sum_{s=1}^{t-1}\ell_{s,i}+m_{t,i}\right)\right)$ by Equation \eqref{radarb_opt}
      $\x_t=\argmin_{\x\in\N_{\Gd}}\sum_{i\in[N]}w_{t,i} d(\x,\x_{t,i})^2$\\
     Observe  $f_t$ and send $\nabla f_t(\cdot)$ to each expert
    }
\end{algorithm2e}


 In Theorem \ref{best-meta} we show the guarantee of the meta algorithm of \radarb and postpone proof details of this section to Appendix \ref{app:radarb}.
 \begin{thm}\label{best-meta}
 Setting  learning rates $\tau=\frac{1}{8NG^2D^2}$ and
 \begin{equation*}
     \begin{split}
         \beta=\min\left\{\sqrt{\frac{2+\ln N}{N(D^2\min\{3(V_T+G^2),\bar{F}_T\}+8G^2D^2\ln 2)} }, \frac{1}{\sqrt{12(D^4L^2+D^2G^2{\zeta}^2)}}\right\}.
     \end{split}
 \end{equation*}

Suppose all losses are $L$-gsc-smooth and non-negative on $\M$. Under Assumptions \ref{hada}, \ref{diam}, \ref{grad},    the regret of the meta algorithm satisfies
 $$
 \textstyle\sumT f_t(\x_t)-\sumT f_t(\x_{t,i})=O\lt(\sqrt{\ln T\min\{V_T,\bar{F}_T\}}\rt)
 $$
 where $N=N^v+N^s$ and $\bar{F}_T=\sumT f_t(\x_t)$.
 \end{thm}

Finally, we show the regret of \radarb is bounded by the smaller of two problem-dependent bounds as follows.

\begin{thm}\label{best-reg}
Suppose all losses are $L$-gsc-smooth and non-negative on $\M$  and allow improper learning. Under Assumptions \ref{hada}, \ref{diam}, \ref{grad}, if we set the set of candidate step sizes as 
\begin{equation}\label{radar_b}
   \H=\H^v\cup\H^s, 
\end{equation}
where $\H^v$ and $\H^s$ are sets of step sizes  in Theorem \ref{var-reg} with $N=N^v$ and Theorem \ref{sml-reg} with $N=N^s$ respectively. Then Algorithm \ref{alg:meta_b} satisfies
{ $$
\text{D-Regret}_T= O\lt(\sqrt{\zeta( P_T(\zeta+1/\delta^2)
+B_T+1)(1+P_T)+\ln T\cdot B_T}\rt)
$$}
where $B_T\coloneqq\min\{V_T,F_T\}$.

\end{thm}
\begin{remark}
Comparing to the result in \cite{zhao2020dynamic}, we find the result of Theorem \ref{best-reg} has an additional $\sqrt{\ln T}$ factor, which comes from our construction of hedging two experts. It will be interesting to remove this  dependence. 
\end{remark}
\section{Conclusion}

In this paper, we consider adaptive online learning on Riemannian manifolds. Equipped with the idea of learning with multiple step sizes and optimistic mirror descent, we derive a series of no-regret algorithms that adapt to  quantities reflecting the intrinsic difficulty  of the online learning problem in different aspects. In the future, it is interesting
to investigate how to achieve optimistic online learning in the proper learning setting. Moving forward, one could further examine whether $\Omega(\log T)$ gradient queries in each round are truly necessary. A curvature-dependent lower bound like the one in \citet{criscitiello2022negative} for Riemannian online optimization also remains open.


\acks{We would like to thank three anonymous referees for constructive comments and suggestions. We gratefully thank the AI4OPT Institute for funding, as part of NSF Award 2112533. We also
acknowledge the NSF for their support through Award IIS-1910077. GW  would like to acknowledge an ARC-ACO fellowship provided by Georgia Tech.}

\bibliography{bib}


\newpage
\appendix





\section{Omitted Proof for Section \ref{radar}}
\label{app:radar}

\subsection{Proof of Theorem \ref{ROGD}}
\label{app:thm:rogd}
We denote $\x_{t+1}'=\expmap_{\x_t}(-\eta\nabla f_t(\x_t))$ and start from the geodesic convexity:
\begin{equation}\label{rogd-eq1}
    \begin{split}
        f_t(\x_t)-f_t(\u_t)\stackrel{\rm (1)}{\leq}&\langle\nabla f_t(\x_t),-\expmap_{\x_t}^{-1}(\u_t)\rangle\\
        =&\frac{1}{\eta}\langle\expmap_{\x_t}^{-1}\x_{t+1}',\expmap_{\x_t}^{-1}\u_t\rangle\\
        \stackrel{\rm (2)}{\leq} &\frac{1}{2\eta}\left(\|\expmap_{\x_t}^{-1}\u_t\|^2-\|\expmap_{\x_{t+1}'}^{-1}\u_t\|^2+{\zeta}\|\expmap_{\x_t}^{-1}\x_{t+1}'\|^2\right)\\
        \stackrel{\rm (3)}{\leq}&\frac{1}{2\eta}\left(\|\expmap_{\x_t}^{-1}\u_t\|^2-\|\expmap_{\x_{t+1}}^{-1}\u_t\|^2\right)+\frac{\eta{\zeta}G^2}{2}\\
        =&\frac{1}{2\eta}\left(\|\expmap_{\x_t}^{-1}\u_t\|^2-\|\expmap_{\x_{t+1}}^{-1}\u_{t+1}\|^2+\|\expmap_{\x_{t+1}}^{-1}\u_{t+1}\|^2-\|\expmap_{\x_{t+1}}^{-1}\u_t\|^2\right)+\frac{\eta{\zeta}G^2}{2}\\
        \stackrel{\rm (4)}{\leq} &\frac{1}{2\eta}\left(\|\expmap_{\x_t}^{-1}\u_t\|^2-\|\expmap_{\x_{t+1}}^{-1}\u_{t+1}\|^2+2D\|\expmap_{\u_t}^{-1}\u_{t+1}\|\right)+\frac{\eta{\zeta}G^2}{2},
    \end{split}
\end{equation}
where for the second inequality we use Lemma \ref{cos1}, while the third is due to Lemma \ref{proj} and Assumption \ref{grad}. For the last inequality, we invoke triangle inequality and Assumption \ref{diam}.

WLOG, we can assume $\u_{T+1}=\u_T$ and sum from $t=1$ to $T$:
\begin{equation}
    \sum_{t=1}^T f_t(\x_t)-f_t(\u_t)\leq \frac{D^2}{2\eta}+\frac{DP_T}{\eta}+\frac{\eta{\zeta}G^2 T}{2}.
\end{equation}

\subsection{Proof of Lemma \ref{meta-rader}}
\label{app:lem:meta-radar}
This is a generalization of \cite[Theorem 2.2]{cesa2006prediction} to the Riemannian manifold. Let $L_{t,i}=\sum_{s=1}^tf_s(\x_{s,i})$ and $W_t=\sum_{i=1}^Nw_{1,i}e^{-\beta L_{t,i}}$. We have the following lower bound for $\ln W_T$,
$$
\ln(W_T)=\ln\lt(\sum_{i\in[N]}w_{1,i}e^{-\beta L_{t,i}}\rt)\geq -\beta\min_{i\in[N]}\lt(L_{T,i}+\frac{1}{\beta}\ln\frac{1}{w_{1,i}}\rt).
$$
For the next step, we try to get an upper bound on $\ln W_T$. When $t\geq 2$, we have
$$
\ln\lt(\frac{W_t}{W_{t-1}}\rt)=\ln\frac{\sum_{i\in[N]}w_{1,i}e^{-\beta L_{t-1,i}}e^{-\beta f_t(\x_{t,i})}}{\sum_{i\in[N]}w_{1,i}e^{-\beta L_{t-1,i}}}=\ln\lt(\sum_{i\in[N]}w_{t,i}e^{-\beta f_t(\x_{t,i})}\rt),
$$
where the updating rule of Hedge 
$$
w_{t,i}=\frac{w_{1,i}e^{-\beta L_{t-1,i}}}{\sum_{j\in[N]}w_{1,j}e^{-\beta L_{t-1,j}}}
$$
is applied.
Therefore
\begin{equation}
    \begin{split}
        &\ln W_T=\ln W_1+\sumsT\ln\lt(\frac{W_t}{W_{t-1}}\rt)=\sumT\ln\lt(\sum_{i\in[N]}w_{t,i}e^{-\beta f_t(\x_{t,i})}\rt)\\
        \leq&\sumT\lt(-\beta\sum_{i\in[N]}w_{t,i}f_t(\x_{t,i})+\frac{\beta^2G^2D^2}{8}\rt)\leq\sumT\lt(-\beta f_t(\x_t)+\frac{\beta^2G^2D^2}{8}\rt),
    \end{split}
\end{equation}
where the first inequality follows from Hoeffding's inequality and $f_t(\x^\star)\leq f_t(\x)\leq f_t(\x^\star)+G\cdot D$ holds for any $\x\in\N$ and $\x^\star=\argmin_{\x\in\N}f_t(\x)$, and the second inequality is due to both the Fr\'echet mean and the geodesic mean satisfy Jensen's inequality. For the Fr\'echet mean, we can apply Lemma \ref{jensen}. While
Lemmas \ref{gsc-mean} and \ref{gsc-com} ensure
the geodesic mean satisfies the requirement.

Combining the lower and upper bound for $\ln W_T$, we see
$$
-\beta\min_{i\in[N]}\lt(L_{T,i}+\frac{1}{\beta}\ln\frac{1}{w_{1,i}}\rt)\leq\sumT\lt(-\beta f_t(\x_t)+\frac{\beta^2G^2D^2}{8}\rt).
$$
After simplifying, we get
$$
\sumT f_t(\x_t)-\min_{i\in[N]}\lt(\sumT f_t(\x_{t,i})+\frac{1}{\beta}\ln\frac{1}{w_{1,i}}\rt)\leq\frac{\beta G^2D^2T}{8}.
$$
Setting $\beta=\sqrt{\frac{8}{G^2D^2T}}$, we have
$$
\sum_{t=1}^Tf_t(\x_t)-\sum_{t=1}^Tf_t(\x_{t,i})\leq\sqrt{\frac{G^2D^2T}{8}}\left(1+\ln\frac{1}{w_{1,i}}\right).
$$

\subsection{Proof of Theorem \ref{RAder}}
\label{app:thm:rader}
Each expert performs R-OGD, so by Theorem \ref{ROGD} we have
\begin{equation}
    \sum_{t=1}^Tf_t(\x_{t,i})-f_t(\u_t)\leq\frac{D^2+2DP_T}{2\eta}+\frac{\eta{\zeta}G^2T}{2}.
\end{equation}
holds for any $i\in [N]$.
Now it suffices to verify that there always exists $\eta_k\in\mathcal{H}$ which is close to the optimal stepsize
\begin{equation}\label{eta_star}
\eta^\star=\sqrt{\frac{D^2+2DP_T}{{\zeta} G^2T}}.
\end{equation}

By Assumption \ref{diam}, 
$$
0 \leq P_{T}=\sum_{t=2}^{T}d(\u_{t-1},\u_t)\leq T D.
$$
Thus
$$
\sqrt{\frac{D^{2}}{TG^{2}{\zeta}}} \leq \eta^{*} \leq \sqrt{\frac{D^{2}+2 T D^{2}}{TG^{2}{\zeta}}}
$$
It is obvious that
$$
\min \mathcal{H}=\sqrt{\frac{D^{2}}{TG^{2}{\zeta}}}, \text { and } \max \mathcal{H} \geq 2\sqrt{\frac{D^{2}+2 T D^{2}}{TG^{2}{\zeta}}}
$$
Therefore, there exists $k\in[N-1]$ such that
\begin{equation}
    \label{etak}
    \eta_{k}=2^{k-1} \sqrt{\frac{D^{2}}{TG^{2}{\zeta}}} \leq \eta^{*} \leq 2 \eta_{k}
\end{equation}
The dynamic regret of the $k$-th expert is

\begin{equation}
    \begin{split}
        \label{expert-rader}
& \sum_{t=1}^{T}f_{t}(\x_{t,k})-\sum_{t=1}^{T} f_{t}\left(\mathbf{u}_{t}\right) \\
\stackrel{\rm (1)}{\leq} & \frac{D^{2}}{2 \eta_{k}}+\frac{DP_T}{\eta_{k}}+\left(\frac{\eta_{k} TG^{2}{\zeta}}{2}\right) \\
\stackrel{\rm (2)}{\leq} & \frac{D^{2}}{\eta^{*}}+\frac{2 DP_T}{\eta^{*}}+\left(\frac{\eta^{*} TG^{2}{\zeta}}{2}\right) \\
{=}& \frac{3}{2} \sqrt{TG^{2}{\zeta}\left(D^{2}+2 D P_{T}\right)}.
    \end{split}
\end{equation}
The second inequality follows from Equation \eqref{etak} and we use Equation \eqref{eta_star} to derive the last equality.

Since the initial weight of the $k$-th expert satisfies
$$
w_{1,k}=\frac{N+1}{k(k+1)N}\geq\frac{1}{(k+1)^2},
$$
the regret of the meta algorithm with respect to the $k$-th expert is bounded by
\begin{equation}\label{eq:meta-rader}
    \sum_{t=1}^Tf_t(\x_t)-\sum_{t=1}^Tf_t(\x_{t,k})\leq\sqrt{\frac{G^2D^2T}{8}}\left(1+2\ln(k+1)\right)
\end{equation}
in view of Lemma \ref{meta-rader}.
Combining Equations \eqref{expert-rader} and \eqref{eq:meta-rader}, we have
\begin{equation}
    \begin{split}
        &\sum_{t=1}^Tf_t(\x_t)-\sum_{t=1}^Tf_t(\u_t)\\
        \leq& \frac{3}{2} \sqrt{TG^{2}{\zeta}\left(D^{2}+2 D P_{T}\right)}+\sqrt{\frac{G^2D^2T}{8}}\left(1+2\ln(k+1)\right)\\
        \leq &\sqrt{2\lt(\frac{9}{4}TG^{2}{\zeta}\left(D^{2}+2 D P_{T}\right)+\frac{G^2D^2T}{8}\left(1+2\ln(k+1)\right)^2 \rt)}\\
        =&O\left(\sqrt{{\zeta}(1+P_T)T}\right),
    \end{split}
\end{equation}
where $\sqrt{a}+\sqrt{b}\leq\sqrt{2(a+b)}$ is applied to derive the second inequality, and for the equality follows from $\ln k=O(\log\log P_T)=o(\sqrt{P_T})$.

\subsection{Minimax Dynamic Regret Lower Bound on Hadamard Manifolds}
\label{app:lb}
In this part, we first establish a $\Omega(\sqrt{T})$ minimax static regret lower bound on Hadamard manifolds following the classical work of \citet{abernethy2008optimal}, then follow the reduction in \cite{zhang2018adaptive} to get its dynamic counterpart. We focus on the manifold of SPD matrices \citep{bhatia2009positive} $\S^n_{++}=\{p:p\in\mathbb{R}^{n\times n},p=p^T \text{and } p\succ 0^{n\times n}\}$, which becomes Hadamard when equipped with the affine-invariant metric $\la U,V\ra_p=tr(p^{-1}Up^{-1}V)$ where $tr(\cdot)$ is the trace operator. The tangent space of $\S^n_{++}$ is $\S^n=\{U:U\in\mathbb{R}^{n\times n},U=U^T\}$. Under the affine-invariant metric, we have 
\begin{equation*}
    \begin{split}
        &\expmap_pU=p^{\frac{1}{2}}\exp\lt(p^{-\frac{1}{2}}Up^{-\frac{1}{2}}\rt)p^{\frac{1}{2}} \\
        &\invexp_pq=p^{\frac{1}{2}}\log\lt(p^{-\frac{1}{2}}qp^{-\frac{1}{2}}\rt)p^{\frac{1}{2}} \\
        &d(p, q)=\sqrt{\sum_{i=1}^n(\log\lambda_i(q^{-\frac{1}{2}}pq^{-\frac{1}{2}}))^2}.
    \end{split}
\end{equation*}


For technical reason, we restrict the manifold to be the manifold of SPD matrices with diagonal entries $\D_{++}^n=\{p:p\in\mathbb{R}^{n\times n},p_{i,i}>0, p \text{ is diagonal}\}$ and its tangent space is $\D^n=\{U:U\in\mathbb{R}^{n\times n}, U \text{ is diagonal} \}$. A key component in the proof is the Busemann function \citep{ballmann2012lectures} on $\D_{++}^n$ equipped with affine-invariant metric has a closed form, which we describe as follows.

        

\begin{defn}\citep{ballmann2012lectures}
    Suppose $\M$ is a Hadamard manifold and $c:[0,\infty)$ is a geodesic ray on $\M$ with $\|\dot{c}(0)\|=1$. Then the Busemann function associated with $c$ is defined as
    $$
    b_c(p)=\lim_{t\rightarrow\infty}(d(p, c(t))-t).
    $$
\end{defn}

Busemann functions enjoy the following useful properties.
\begin{lemma}\citep{ballmann2012lectures}\label{buse}
    A Busemann function $b_c$ satisfies
    \begin{enumerate}[label=\arabic*)]
        \item $b_c$ is gsc-convex;
        \item $\nabla b_c(c(t))=\dot{c}(t)$ for any $t\in[0,\infty)$;
        \item $\|\nabla b_c(p)\|\leq 1$ for every $p\in \M$.
    \end{enumerate}
\end{lemma}

\begin{lemma}\citep[Chapter \RomanNumeralCaps{2}.10]{bridson2013metric}
    On $\D_{++}^n$, suppose $c(t)=\expmap_{I}(tX)$, $p=\expmap_{I}(Y)$ and $\|X\|=1$, then the Busemann function is
    $$
    b_c(p)=-\text{tr}(XY).
    $$
\end{lemma}
\begin{proof}
    We first compute
    \begin{equation}
        \begin{split}
            d(p, c(t))^2=&d(\expmap_{I}(Y), \expmap_{I}(tX))^2\\
            =&d(e^Y, e^{tX})^2\\
            =&\sum_{i=1}^n(\log \lambda_i(e^{-t X+Y}))^2\\
            =&d(I, e^{Y-tX})^2\\
            =&d(I, \expmap_{I}(Y-tX))^2\\
            =&\|Y-tX\|^2\\
            =&tr((Y-tX)(Y-tX))\\
            =&tr(Y^2)-2t\cdot tr(XY)+t^2\cdot tr(X^2),
        \end{split}
    \end{equation}
where we use facts that $\expmap_{I}(X)=e^X$ and $d(p, q)=\sqrt{\sum_{i=1}^n(\log\lambda_i(q^{-\frac{1}{2}}pq^{-\frac{1}{2}}))^2}$.
Meanwhile,
$$
\lim_{t\rightarrow\infty}\frac{d(p, c(t))}{t}=1
$$
due to the triangle inequality on $\triangle (p, c(0), c(t))$. Therefore,
\begin{equation*}
    \begin{split}
        b_c(p)=&\lim_{t\rightarrow\infty}(d(p, c(t))-t)\\
        =&\lim_{t\rightarrow\infty}\frac{d(p, c(t))^2-t^2}{2t}\\
        =&-tr(XY).
    \end{split}
\end{equation*}
\end{proof}
\begin{remark}
    We consider $\D_{++}^n$ to ensure $X$ and $Y$ commute thus also $e^Y$ and $e^{tX}$ commute. This is necessary to get a closed form of the Busemann function.
\end{remark}


 Now we describe the minimax game on $\D_{++}^n$. Each round $t$, the player chooses $p_t$ from $\N=\{p:d(p,I)\leq\frac{D}{2}\}=\{p:p=\expmap_{I}(Y),\|Y\|\leq\frac{D}{2}\}$, and the adversary is allowed to pick a geodesic $c^t$, which determines a loss function in
\begin{equation}
    \begin{split}
        \F_t=&\{\alpha_t G_t b_c(p): \|\dot{c}(0)\|=1, \alpha_t\in[0,1] \}\\
        =&\{\alpha_t G_t \cdot -tr(X_tY): \|X_t\|=1, \alpha_t\in[0,1] \}\\
        =&\{-tr(X_tY): \|X_t\|\leq G_t \}\\
    \end{split}
\end{equation}
The domain is gsc-convex by Lemmas \ref{gsc-convex} and \ref{levelset}. Each loss function is gsc-convex and has a gradient upper bound $G_t$ using the second item of Lemma \ref{buse}. WLOG, we assume $D=2$, and the value of the game is

\[\V_T(\N, \{\F_t\})\coloneqq\inf_{\|Y_1\| \leq 1}\sup_{\|X_1\| \leq G_1}\dots \inf_{\|Y_T\| \leq 1}\sup_{\|X_T\| \leq G_T} \lt[ \sumT -tr(X_tY_t)-\inf_{\|Y\|\leq 1}\sumT -tr(X_tY)\rt]\]

\begin{lemma}\label{lb1}
    The value of the minimax game $\V_T$ can be written as
    \[\V_T(\N, \{\F_t\})=\inf_{\|Y_1\| \leq 1}\sup_{\|X_1\| \leq G_1}\dots \inf_{\|Y_T\| \leq 1}\sup_{\|X_T\| \leq G_T} \lt[ \sumT -tr(X_tY_t)+\lt\|\sumT X_t\rt\|\rt]\]
\end{lemma}
\begin{proof}
This is obvious due to Cauchy–Schwarz inequality:
    \[\inf_{\|Y\|\leq 1}\sumT -tr(X_tY)=-\sup_{\|Y\|\leq 1}tr\lt(\sumT X_t Y\rt)=-\lt\|\sumT X_t\rt\|.\]
\end{proof}
\begin{lemma}\label{lb2}
    For $n>2$, the adversary guarantees at least $\sqrt{\sumT G_t^2}$ regardless of the player's strategy.
\end{lemma}
\begin{proof}
    Each round, after the player chooses $Y_t$, the adversary chooses $X_t$ such that $\|X_t\|=G_t$, $\la X_t,Y_t\ra=0$ and $\la X_t, \sum_{s=1}^{t-1}X_s\ra =0$. This is always possible when $n>2$. Under this strategy, $\sumT-tr(X_tY_t)=0$ and we can show $\lt\|\sumT X_t\rt\|=\sqrt{\sumT G_t^2}$ by induction. The case for $T=1$ is obvious. Assume $\lt\|\sum_{s=1}^{t-1}X_s\rt\|=\sqrt{\sum_{s=1}^{t-1}G_s^2}$, then
    $$
    \lt\|\sum_{s=1}^tX_s\rt\|=\lt\|\sum_{s=1}^{t-1}X_s+X_t\rt\|=\sqrt{\lt\|\sum_{s=1}^{t-1}X_s\rt\|^2+\|X_t\|^2}=\sqrt{\sum_{s=1}^tG_t^2}.
    $$
    where the second equality is due to $\la X_t, \sum_{s=1}^{t-1}X_s\ra =0$.
\end{proof}
\begin{lemma}\label{lb3}
    Let $X_0=0$.  If the player plays
    $$
    Y_t=\frac{\sum_{s=1}^{t-1}X_s}{\sqrt{\lt\|\sum_{s=1}^{t-1}X_s\rt\|^2+\sum_{s=t}^TG_s^2}},
    $$
    then
    \[\sup_{\|X_1\| \leq G_1}\sup_{\|X_2\| \leq G_2}\dots \sup_{\|X_T\| \leq G_T} \lt[ \sumT -tr(X_tY_t)+\lt\|\sumT X_t\rt\|\rt]\leq\sqrt{\sumT G_t^2}.\]
\end{lemma}
\begin{proof}
    Let $\Gamma_t^2=\sum_{s=t}^TG_s^2$, $\Gamma_{T+1}=0$ and $\Xt_t=\sum_{s=1}^tX_s$. We  define
    $$
    \Phi_t(X_1,\dots, X_{t-1})=\sum_{s=1}^{t-1}-tr(X_sY_s)+\sqrt{\lt\|\sum_{s=1}^{t-1}X_s\rt\|^2+\Gamma_t^2}
    $$
    and $\Phi_1=\sqrt{\sumT G_t^2}$. We further let
    $$
    \Psi_t(X_1,\dots, X_{t-1})=\sup_{\|X_t\| \leq G_t}\dots \sup_{\|X_T\| \leq G_T} \lt[ \sum_{s=1}^T -tr(X_sY_s)+\lt\|\sum_{s=1}^T X_s\rt\|\rt]
    $$
    be the payoff of the adversary when he plays $X_1,\dots, X_{t-1}$ and then plays optimally.

    We do backward induction for this argument, which means for all $t\in\{1,\dots,T+1\}$,
    $$
    \Psi_t(X_1,\dots,X_{t-1})\leq\Phi_t(X_1,\dots,X_{t-1}).
    $$
    The case of $t=T+1$ is obvious because $\Psi_{T+1}=\Phi_{T+1}$. Assume the argument holds for $t+1$ and we try to show the case for $t$.
    \begin{equation}
    \begin{split}      
        &\Psi_t(X_1,\dots,X_{t-1})\\
        =&\sup_{\|X_t\|\leq G_t}\Psi_{t+1}(X_1,\dots,X_t)\\
        \leq&\sup_{\|X_t\|\leq G_t}\Phi_{t+1}(X_1,\dots,X_t)\\
        =&\sum_{s=1}^{t-1}-tr(X_sY_s)+\sup_{\|X_t\|\leq G_t}\lt[-tr(X_tY_t)+\sqrt{\lt\|\sum_{s=1}^tX_s\rt\|^2+\Gamma_{t+1}^2}\rt].
    \end{split}
    \end{equation}
    Now it suffices to show
    $$
    \sup_{\|X_t\|\leq G_t}\lt[-tr(X_tY_t)+\sqrt{\lt\|\sum_{s=1}^tX_s\rt\|^2+\Gamma_{t+1}^2}\rt]\leq \sqrt{\lt\|\sum_{s=1}^{t-1}X_s\rt\|^2+\Gamma_t^2}
    $$
    to establish our argument. Recall that
    $$
    Y_t=\frac{\sum_{s=1}^{t-1}X_s}{\sqrt{\lt\|\sum_{s=1}^{t-1}X_s\rt\|^2+\sum_{s=t}^TG_s^2}},
    $$
    and denote $\Xt_{t-1}=\sum_{s=1}^{t-1}X_s$. It turns out that what we need to show is
    $$
    \sup_{\|X_t\|\leq G_t} -tr\lt(\frac{\la X_t, \Xtt\ra}{\sqrt{\|\Xtt\|^2+\Gamma_t^2}}\rt)+\sqrt{\|\Xtt+X_t\|^2+\Gamma_{t+1}^2}\leq\sqrt{\|\Xtt\|^2+\Gamma_t^2}.
    $$
    We use the Lagrange multiplier method to prove this argument. Let
    $$
    g(X_t)=-tr\lt(\frac{\la X_t, \Xtt\ra}{\sqrt{\|\Xtt\|^2+\Gamma_t^2}}\rt)+\sqrt{\|\Xtt+X_t\|^2+\Gamma_{t+1}^2}+\lambda(\|X_t\|^2-G_t^2),
    $$
    then the stationary point of $g$ satisfies
    $$
    \frac{\partial g(X_t)}{\partial X_t}=-\frac{\Xtt}{\sqrt{\|\Xtt\|^2+\Gamma_t^2}}+\frac{\Xtt+X_t}{\sqrt{\|\Xtt+X_t\|^2+\Gamma_{t+1}^2}}+2\lambda X_t=0
    $$
    and
    $$\lambda(\|X_t\|^2-G_t^2)=0.$$
    We first consider that $\Xtt$ is co-linear with $X_t$.
    When $\lambda=0$, we have $X_t=c\Xtt$ where
    $$
    c = \frac{\Gamma_{t+1}}{\Gamma_t}-1.
    $$
    

    If $\Xtt$ is co-linear with $X_t$ and $\lambda\neq 0$, we know $\|X_t\|=G_t$ and again let $X_t=G_t\frac{\Xtt}{\|\Xtt\|}$ or $X_t=-G_t\frac{\Xtt}{\|\Xtt\|}$. Then we need to ensure
    $$
    g(c\Xtt)\leq\sqrt{\|\Xtt\|^2+\Gamma_t^2}
    $$
    holds for $c = \frac{\Gamma_{t+1}}{\Gamma_t}-1,\frac{G_t}{\|\Xtt\|}$ and $-\frac{G_t}{\|\Xtt\|}$.

    By Lemma \ref{auxi}, it suffices to verify
    $$
    (c^2-1)\|\Xtt\|^2\Gamma_t^2+(\|\Xtt\|^2+\Gamma_t^2)\Gamma_{t+1}^2\leq\Gamma_t^4.
    $$
    If $c=\frac{\Gamma_{t+1}}{\Gamma_t}-1$, we have to ensure
    \begin{equation}
        \begin{split}
           & (c^2-1)\|\Xtt\|^2\Gamma_t^2+(\|\Xtt\|^2+\Gamma_t^2)\Gamma_{t+1}^2-\Gamma_t^4\\
           =&\lt(\frac{\Gtt^2}{\Gt^2}-2\frac{\Gtt}{\Gt}\rt)\|\Xtt\|^2\Gt^2+\|\Xtt\|^2\Gtt^2+\Gt^2\Gtt^2-\Gt^4\\
           =&2(\Gtt-\Gt)\Gtt\|\Xtt\|^2+\Gt^2(\Gtt^2-\Gt^2)\leq 0.
        \end{split}
    \end{equation}
    For the case where $c^2=\frac{G_t^2}{\|\Xtt\|^2}$, we have
    \begin{equation}
        \begin{split}
           & (c^2-1)\|\Xtt\|^2\Gamma_t^2+(\|\Xtt\|^2+\Gamma_t^2)\Gamma_{t+1}^2-\Gamma_t^4\\
           =&\lt(\frac{G_t^2}{\|\Xtt\|^2}-1\rt)\|\Xtt\|^2\Gt^2+(\|\Xtt\|^2+\Gt^2)\Gtt^2-\Gt^4\\
           =&\Gt^2(G_t^2+\Gtt^2-\Gt^2)-G_t^2\|\Xtt\|^2\\
           =&-G_t^2\|\Xtt\|^2\leq 0.
        \end{split}
    \end{equation}
    The only case left is when $X_t$ is not parallel to $\Xtt$. 
    $\lambda=0$ implies $X_t=0$ and thus
    $$
    g(0)=\sqrt{\|\Xtt\|^2+\Gtt^2}\leq\sqrt{\|\Xtt\|^2+\Gt^2}.
    $$
    If $\lambda\neq 0$ then $\|X_t\|=G$. We  have
    $$
    -\frac{\Xtt}{\sqrt{\|\Xtt\|^2+\Gamma_t^2}}+\frac{\Xtt}{\sqrt{\|\Xtt+X_t\|^2+\Gtt^2}}=0
    $$
    which in turn implies $\la X_t,\Xtt\ra=0$. This is the maximum point of $g$ as now 
    $$
    g(X_t)=\sqrt{\|\Xtt\|^2+\Gt^2}.
    $$
  Thus we finished the induction step and the lemma was established.  
    
\end{proof}

\begin{lemma}\label{auxi}
    $$
     -tr\lt(\frac{\la X_t, \Xtt\ra}{\sqrt{\|\Xtt\|^2+\Gamma_t^2}}\rt)+\sqrt{\|\Xtt+X_t\|^2+\Gamma_{t+1}^2}\leq\sqrt{\|\Xtt\|^2+\Gamma_t^2}.
    $$ holds for $X_t=c \Xtt$ iff

    $$
    (c^2-1)\|\Xtt\|^2\Gamma_t^2+(\|\Xtt\|^2+\Gamma_t^2)\Gamma_{t+1}^2\leq\Gamma_t^4.
    $$
\end{lemma}
\begin{proof}
    The statement we want to show is
    $$
    -\frac{c\|\Xtt\|^2}{\sqrt{\|\Xtt\|^2+\Gamma_t^2}}+\sqrt{\|\Xtt\|^2(1+c)^2+\Gamma_{t+1}^2}\leq \sqrt{\|\Xtt\|^2+\Gamma_t^2}.
    $$
    Let $\alpha=\|\Xtt\|^2$, $\beta=\Gamma_t^2$ and $\gamma=\Gamma_{t+1}^2$. Following a series of algebraic manipulations, we get
    $$
    (c^2-1)\alpha\beta+(\alpha+\beta)\gamma\leq \beta^2.
    $$
    And the argument is proved after plugging back $\alpha, \beta, \gamma$.
\end{proof}
\begin{thm}\label{lb_static}
There exists a  game on $\D_{++}^n$ such that we can exactly compute the value of the minimax regret. Specifically,
the decision set of the player is $\N=\{p:p=\expmap_{I}(Y),\|Y\|\leq\frac{D}{2}\}$, and the adversary is allowed to pick a  loss function in
$$
\F_t=\{\alpha_t G_t b_c(p): \|\dot{c}(0)\|=1, \alpha_t\in[0,1] \}=\{-tr(X_tY): \|X_t\|\leq G_t \}\\.
$$
Then the minimax value of the game is
$$
\V_T(\N,\{\F_t\})=\frac{D}{2}\sqrt{\sumT G_t^2}.
$$
In addition, the optimal strategy of the player is

$$
    Y_t=\frac{\sum_{s=1}^{t-1}X_s}{\sqrt{\lt\|\sum_{s=1}^{t-1}X_s\rt\|^2+\sum_{s=t}^TG_s^2}}.
    $$
\end{thm}
\begin{proof}
    The proposition is a direct conclusion of Lemmas \ref{lb1}, \ref{lb2}, and \ref{lb3}.
\end{proof}
\begin{thm}
    There exists a comparator sequence which satisfies $\sumsT d(\u_t,\u_{t-1})\leq P_T$ and
    the dynamic minimax regret lower bound on Hadamard manifolds is $\Omega(G\sqrt{D^2+DP_T})$.
\end{thm}
\begin{proof}
    We combine Theorem \ref{lb_static} with a reduction in  \cite{zhang2018adaptive} to finish the proof. By Theorem \ref{lb_static} we have
    $$
    \V_T=\inf_{\|Y_1\| \leq 1}\sup_{\|X_1\| \leq G}\dots \inf_{\|Y_T\| \leq 1}\sup_{\|X_T\| \leq G} \lt[ \sumT -tr(X_tY_t)-\inf_{\|Y\|\leq 1}\sumT -tr(X_tY)\rt]=\frac{GD\sqrt{T}}{2}.
    $$
    Note that the path-length is upper bounded by $TD$. For any $\tau\in[0, TD]$, we define the set of comparators with path-length bounded by $\tau$ as
    $$
    C(\tau)=\lt\{\u_1,\dots,\u_T\in\N:\sumsT d(\u_t,\u_{t-1})\leq\tau\rt\}
    $$
    where $\N=\{\u:d(I, \u)\leq\frac{D}{2}\}$ is a gsc-convex subset   and the minimax dynamic regret w.r.t. $C(\tau)$ is
    $$
     \V_T(C(\tau))=\inf_{\small\|Y_1\| \leq 1}\sup_{\small\|X_1\| \leq G}\dots \inf_{\small\|Y_T\| \leq 1}\sup_{\small\|X_T\| \leq G} \lt[ \sumT -tr(X_tY_t)-\inf_{\u_1,\dots,\u_T\in C(\tau)}\sumT -tr(X_t\invexp_I \u_t)\rt].
    $$
    We distinguish two cases. When $\tau\leq D$, we invoke the minimax static regret directly to get
    \begin{equation}\label{eq-lb1}
            \V_T(C(\tau))\geq \V_T=\frac{GD\sqrt{T}}{2}.
    \end{equation}
    For the case of $\tau\geq D$, WLOG, we assume $\lc \tau/D\rceil$ divides $T$ and let $L$ be the quotient. We construct a subset of $C(\tau)$, named $C'(\tau)$, which contains comparators that are fixed for each consecutive $L$ rounds. Specifically,
    $$
    C'(\tau)=\lt\{\u_1,\dots,\u_T\in\N:\u_{(i-1)L+1}=\dots\u_{i L},\forall i\in [1,\lc\tau/D\rc]\rt\}.
    $$
    Note that the path-length of comparators in $C'(\tau)$ is at most $(\lc\tau/D\rc-1)D\leq \tau$, which  implies $C'(\tau)$ is a subset of $C(\tau)$. Thus we have
    \begin{equation}\label{eq-lb2}
        \V_T(C(\tau))\geq \V_T(C'(\tau)).
    \end{equation}
    The objective of introducing $C'(\tau)$ is we can set $\u_{(i-1)L+1}=\dots=\u_{(iL)}$ to be the offline minimizer of the $i$-th segment and invoke the minimax lower bound for the static regret for each segment. Thus we have
    \begin{equation}\label{eq-lb3}
    \begin{split}      
        &\V_T(C'(\tau))\\
        =&\inf_{\small\|Y_1\| \leq 1}\sup_{\small\|X_1\| \leq G}\dots \inf_{\small\|Y_T\| \leq 1}\sup_{\small\|X_T\| \leq G} \lt[ \sumT -tr(X_tY_t)-\inf_{\u_1,\dots,\u_T\in C'(\tau)}\sumT -tr(X_t\invexp_I \u_t)\rt]\\
        =&\inf_{\small\|Y_1\| \leq 1}\sup_{\small\|X_1\| \leq G}\dots \inf_{\small\|Y_T\| \leq 1}\sup_{\small\|X_T\| \leq G} \lt[ \sumT -tr(X_tY_t)-\sum_{i=1}^{\lc\tau/D\rc}\inf_{\|Y\|\leq 1}\sum_{t=(i-1)L+1}^{iL} -tr(X_tY)\rt]\\
        =&\lc\tau/D\rc\frac{GD\sqrt{L}}{2}=\frac{GD\sqrt{T\lc\tau/D\rc}}{2}\geq\frac{G\sqrt{TD\tau}}{2}.
    \end{split}
    \end{equation}
    Combining Equations \eqref{eq-lb1}, \eqref{eq-lb2} and \eqref{eq-lb3} yields
    $$
    \V_T(C(\tau))\geq \frac{G}{2}\max(D\sqrt{T},\sqrt{TD\tau})=\Omega(G\sqrt{T(D^2+D\tau)}).
    $$
\end{proof}
\section{Omitted Proof for Section \ref{radarv}}
\label{app:radarv}
\subsection{Proof of Theorem \ref{var-expert}}
\label{app:thm:var-expert}



We first argue $\N_c$ is gsc-convex for any $c\geq 0$ to ensure the algorithm is well-defined. By Lemma \ref{gsc-convex}, $d(\x, \N)$ is gsc-convex on Hadamard manifolds. The sub-level set of a gsc-convex function is a gsc-convex set due to Lemma \ref{levelset}, which implies $\N_c$ is gsc-convex.
We notice that
$$
f_t(\x_t)-f_t(\u_t)=(f_t(\x_t)-f_t(\x_t'))+(f_t(\x_t')-f_t(\u_t))
$$
and derive upper bounds for two terms individually. If $\x_t'\in\ndm$ then $\x_t'=\x_t$ and $f_t(\x_t)-f_t(\x_t')=0$. If this is not the case, by Lemma \ref{proj}, we have
$$
d(\x_t', \x_t)\leq d(\x_t',\z)\leq d(\x_t', \y_t)$$ where $\z$ is the intersection of $\N_{\delta}$ and the geodesic segment connecting $\x_t'$ and $\y_t$. Thus
\begin{equation}\label{var-exp-eq0}
    f_t(\x_t)-f_t(\x_t')\leq\la\nabla f_t(\x_t),-\invexp_{\x_t}{\x_t'}\ra\leq\|\nabla f_t(\x_t)\|\cdot d(\x_t',\x_t)\leq G\cdot d(\x_t', \y_t),
\end{equation}
where we notice $\x_t\in\ndm$ and use Assumption \ref{grad}.
Let $\y_{t+1}'=\expmap_{\x_t'}\left(-\eta\nabla f_t(\x_t')+\expmap_{\x_t'}^{-1}(\y_t)\right)$. The second term $f_t(\x_t')-f_t(\u_t)$ can be bounded by

\begin{equation}\label{var-exp-eq1}
    \begin{split}
        &f_t(\x_t')-f_t(\u_t)\stackrel{\rm (1)}{\leq} -\langle\expmap_{\x_t'}^{-1}(\u_t),\nabla f_t(\x_t')\rangle\\
        =&\frac{1}{\eta}\langle\expmap_{\x_t'}^{-1}(\u_t),\expmap_{\x_t'}^{-1}(\y_{t+1}')-\expmap_{\x_t'}^{-1}(\y_t)\rangle\\
        \stackrel{\rm (2)}{\leq}&\frac{1}{2\eta}\left({\zeta}d(\x_t',\y_{t+1}')^2+d(\x_t',\u_t)^2-d(\y_{t+1}',\u_t)^2\right)-\frac{1}{2\eta}\left(d(\x_t',\y_t)^2+d(\x_t',\u_t)^2-d(\y_t,\u_t)^2\right)\\
        =&\frac{1}{2\eta}\left({\zeta}d(\x_t',\y_{t+1}')^2-d(\x_t',\y_t)^2+d(\y_t,\u_t)^2-d(\y_{t+1}',\u_t)^2\right)\\
        =&\frac{1}{2\eta}\left({\zeta}\|-\eta\nabla f_t(\x_t')+\expmap_{\x_t'}^{-1}(\y_t)\|^2
        -d(\x_t',\y_t)^2+d(\y_t,\u_t)^2-d(\y_{t+1}',\u_t)^2\right)\\
        =&\frac{1}{2\eta}\left({\zeta}\|-\eta\nabla f_t(\x_t')+\eta\Ga_{\y_t}^{\x_t'}M_t\|^2
        -d(\x_t',\y_t)^2+d(\y_t,\u_t)^2-d(\y_{t+1}',\u_t)^2\right)\\
        \stackrel{\rm (3)}{\leq}&\frac{1}{2\eta}\left({\zeta}\|-\eta\nabla f_t(\x_t')+\eta\Ga_{\y_t}^{\x_t'}M_t\|^2
        -d(\x_t',\y_t)^2+d(\y_t,\u_t)^2-d(\y_{t+1},\u_t)^2\right)
    \end{split}
\end{equation}
where the second inequality follows from Lemmas \ref{cos1} and \ref{cos2} and  the third one is due to the non-expansive property of projection onto Hadamard manifolds. We apply $\Ga_{\x}^{\y}\expmap_{\x}^{-1}\y=-\expmap_{\y}^{-1}\x$ to derive the last equality.

Now we can get the desired squared term $\|\nabla f_t(\y_t)-M_t\|^2$ by considering
\begin{equation}\label{var-exp-eq2}
    \begin{split}
        &\|\nabla f_t(\x_t')-\Ga_{\y_t}^{\x_t'}M_t\|^2\\
        =&\|\nabla f_t(\x_t')-\Ga_{\y_t}^{\x_t'}\nabla f_t(\y_t)+\Ga_{\y_t}^{\x_t'}\nabla f_t(\y_t)-\Ga_{\y_t}^{\x_t'}M_t\|^2\\
        \leq& 2\left(\|\nabla f_t(\x_t')-\Ga_{\y_t}^{\x_t'}\nabla f_t(\y_t)\|^2+\|\Ga_{\y_t}^{\x_t'}\nabla f_t(\y_t)-\Ga_{\y_t}^{\x_t'}M_t\|^2\right)\\
        \leq &2L^2d(\x_t',\y_t)^2+2\|\nabla f_t(\y_t)-M_t\|^2,
    \end{split}
\end{equation}
 where in the first inequlity we use $\|\x+\y\|^2\leq 2\|\x\|^2+2\|\y\|^2$ holds for any SPD norm $\|\cdot\|$, and the second inequality is due to the smoothness of $f$ and parallel transport is an isometry. Combining Equations \eqref{var-exp-eq0}, \eqref{var-exp-eq1} and \eqref{var-exp-eq2}, we have
\begin{equation}\label{var-exp-eq3}
    \begin{split}
        &f_t(\x_t)-f_t(\u_t)\\
        \leq&G d(\x_t', \y_t)+\frac{\eta{\zeta}}{2}\left(2\|\nabla f_t(\y_t)-M_t\|^2+2L^2d(\x_t',\y_t)^2\right)\\
        & -\frac{1}{2\eta}d(\x_t',\y_t)^2+\frac{1}{2\eta}(d(\y_t,\u_t)^2-d(\y_{t+1},\u_t)^2)\\
        \leq&\eta{\zeta}\|\nabla f_t(\y_t)-M_t\|^2+\frac{1}{2\eta}\left(2\eta G+2\eta^2{\zeta}L^2d(\x_t',\y_t)-d(\x_t',\y_t)\right)d(\x_t',\y_t)\\
        & +\frac{1}{2\eta}(d(\y_t,\u_t)^2-d(\y_{t+1},\u_t)^2).
    \end{split}
\end{equation}
Now we show
\begin{equation}\label{key}
    \frac{1}{2\eta}\left(2\eta G+2\eta^2{\zeta}L^2d(\x_t',\y_t)-d(\x_t',\y_t)\right)d(\x_t',\y_t)\leq 0
\end{equation}
holds for any $t\in[T]$. First we consider the case that $d(\x_t',\y_t)\leq \delta M$, which means $\x_t'\in\ndm$ and $f_t(\x_t)=f_t(\x_t')$. Thus Equation \eqref{key} is implied by
$$
2\eta^2{\zeta}L^2d(\x_t',\y_t)-d(\x_t',\y_t)\leq 0,$$
which is obviously true by considering our assumption on $\eta$:
$$\eta\leq\frac{\delta M}{G+(G^2+2\zeta\delta^2M^2L^2)^{\frac{1}{2}}}\leq\frac{1}{\sqrt{2\zeta L^2}}.$$


When $d(\x_t', \y_t)\geq\delta M$, to simplify the proof, we denote $\lambda=d(\x_t', \y_t)$ and try to find $\eta$ such that
 \begin{equation}\label{quadratic}
      h(\eta;\lambda)\coloneqq 2\eta G+2\eta^2\zeta L^2\lambda-\lambda\leq 0
 \end{equation}
 holds for any $\lambda\geq \delta M$. We denote the only non-negative root of $h(\eta;\lambda)$ as $\eta(\lambda)$, which can be solved explicitly as
$$
\eta(\lambda)=\frac{-G+(G^2+2\zeta \lambda^2L^2)^{\frac{1}{2}}}{2\zeta\lambda L^2}.
$$

Applying Lemma \ref{lem:tech} with $a=G$ and $b=2\zeta L^2$, we know $\eta(\lambda)$ increases on $[0,\infty)$. Thus $\eta(\lambda)\geq\eta(\delta M)$ holds for any $\lambda\geq\delta M$. Combining with the fact that $h(0;\lambda)=-\lambda<0$, we know $h(\eta;\lambda)\leq 0$ holds for any $\eta\leq\eta(\lambda)$, so we can simply set
$$\eta\leq \min_{\lambda} \eta(\lambda)=\eta\lt(\delta M\rt)=\frac{\delta M}{G+(G^2+2\zeta\delta^2M^2L^2)^{\frac{1}{2}}}.
$$
to ensure $h(\eta;\lambda)\leq 0$ always holds.

Now it suffices to bound
\begin{equation}
\begin{split}
    &\sum_{t=1}^T\frac{1}{2\eta}(d(\y_t,\u_t)^2-d(\y_{t+1},\u_t)^2)\\
    \leq& \frac{d(\y_1,\u_1)^2}{2\eta}+\sum_{t=2}^T\frac{\left(d(\y_t,\u_t)^2-d(\y_t,\u_{t-1})^2\right)}{2\eta}\\
    \leq&\frac{D^2}{2\eta}+2D\frac{\sum_{t=2}^Td(\u_t,\u_{t-1})}{2\eta}=\frac{D^2+2DP_T}{2\eta}.
\end{split}
\end{equation}

Finally, we apply the telescoping-sum  on Equation \eqref{var-exp-eq3},
\begin{equation}
    \begin{split}
        \sum_{t=1}^Tf_t(\x_t)-\sum_{t=1}^Tf_t(\u_t)\leq\eta{\zeta}\sumT\|\nabla f_t(\y_t)-M_t\|^2+\frac{D^2+2DP_T}{2\eta}.
    \end{split}
\end{equation}


\subsection{Extension of Theorem \ref{var-expert} to CAT\texorpdfstring{$(\kappa)$}{ a} Spaces}
\label{cat-proof}
In this part, we show how to get optimistic regret bound on CAT$(\kappa)$ spaces, the sectional curvature of which is upper bounded by $\kappa$. Note that Hadamard manifolds are complete CAT$(0)$ spaces. To proceed, we make the following assumption.
\begin{ass}\label{cat}
    The sectional curvature of  manifold $\M$ satisfies $-\kappa_1\leq\kappa\leq\kappa_2$ where $\kappa_1\geq 0$. We define
    $$
    \D(\kappa)\coloneqq \begin{cases}
\infty,\quad &\kappa\leq 0 \\
\frac{\pi}{\sqrt{\kappa}},\quad &\kappa>0.
\end{cases} $$
The diameter of the gsc-convex set $\N\subset \M$ is $D$ and we assume $D+2\delta M\leq \D(\kappa_2)$. The gradient satisfies $\sup_{\x\in\ndm}\|\nabla f_t(\x)\|\leq G.$
\end{ass}
\begin{lemma}\cite[Corollary 2.1]{alimisis2020continuous}\label{cos3}
    Let $\M$ be a Riemannian manifold with sectional curvature upper bounded by $\kappa_2$ and $\N\subset\M$ be a gsc-convex set with diameter upper bounded by $\D(\kappa_2)$. For a geodesic triangle fully lies within $\N$ with side lengths $a, b, c$, we have
$$
a^{2} \geq {\xi}(\kappa_2, D) b^{2}+c^{2}-2 b c \cos A
$$
where ${\xi}(\kappa_2, D)=\sqrt{-\kappa_2} D \operatorname{coth}(\sqrt{-\kappa_2} D)$ when $\kappa_2\leq 0$ and ${\xi}(\kappa_2, D)=\sqrt{\kappa_2} D \operatorname{cot}(\sqrt{\kappa_2} D)$ when $\kappa_2> 0$.
\end{lemma}
\begin{defn}
    We define $\zeta=\zeta(-\kappa_1, D+2\delta M)$ and $\xi=\xi(\kappa_2, D+2\delta M)$ where $\xi(\cdot,\cdot)$ and $\zeta(\cdot,\cdot)$ are defined in Lemmas \ref{cos3} and \ref{cos1}, respectively.
\end{defn}
\begin{lemma}\citep{bridson2013metric}\label{cat2}
    For a CAT$(\kappa)$ space $\M$, a ball of diameter smaller than $\D(\kappa)$ is convex. Let $C$ be a convex subset in $\M$. If $d(\x, C)\leq \frac{\D(\kappa)}{2}$ then $d(\x, C)$ is convex and there exists a unique point $\Pi_C(\x)\in\C$ such that $d(\x,\Pi_C(\x))=d(\x,C)=\inf_{\y\in C}d(\x,\y)$.
\end{lemma}
\begin{thm}
Suppose all losses $f_t$ are  $L$-gsc-smooth on $\M$. Under Assumptions  \ref{cat},  the iterates  
\begin{equation}
    \begin{split}
        \textstyle &\x_{t}'=\expmap_{\y_t}(-\eta M_t)\\
        &\x_t=\Pi_{\ndm}\x_t'\\
        &\y_{t+1}=\Pi_{\N}\expmap_{\x_t'}\left(-\eta\nabla f_t(\x_t')+\expmap_{\x_t'}^{-1}(\y_t)\right).
    \end{split}
\end{equation}
 satisfies
 \begin{equation*}
    \begin{split}
        \textstyle\sum_{t=1}^Tf_t(\x_t)-\sum_{t=1}^Tf_t(\u_t)\leq\eta{\zeta}\sumT\|\nabla f_t(\y_t)-M_t\|^2+\frac{D^2+2DP_T}{2\eta}.
    \end{split}
\end{equation*}
for any  $\u_1,\dots,\u_T\in\N$ and $\eta\leq\min\lt\{\frac{\xi\delta M}{G+(G^2+2\zeta\xi\delta^2M^2L^2)^{\frac{1}{2}}},\frac{\D(\kappa_2)}{2(G+2M)}\rt\}$.
\end{thm}
\begin{proof}
    We highlight key differences between the proof of Theorem \ref{var-expert}. Again we let $\y_{t+1}'=\expmap_{\x_t'}\left(-\eta\nabla f_t(\x_t')+\expmap_{\x_t'}^{-1}(\y_t)\right)$. First, we need to argue the algorithm is well-defined. The diameter of $\ndm$ is at most $D+2\delta M$ by triangle inequality, so $\ndm$ is gsc-convex by Assumption \ref{cat} and Lemma \ref{cat2}. We also need to ensure that $d(\x_t',\ndm)\leq\frac{\dk}{2}$ and $d(\y_{t+1}',\N)\leq\frac{\dk}{2}$ and apply Lemma \ref{cat2} to show the projection is unique and non-expansive. For $d(\x_t',\ndm)$, we have
    $$
    d(\x_t',\ndm)\leq  d(\x_t',\y_t)\leq \eta M\leq\frac{\D(\kappa_2)}{2}.
    $$
    by $\eta\leq\frac{\D(\kappa_2)}{2(G+2M)}$. Similarly, for $d(\y_{t+1}',\N)$
    \begin{equation*}
        \begin{split}
            &d(\y_{t+1}',\N)\leq d(\y_{t+1}',\y_t)\leq d(\y_{t+1}',\x_t')+d(\y_{t},\x_t')\\
            \leq& \|-\eta\nabla f_t(\x_t')+\eta\Ga_{\y_t}^{\x_t'}M_t\|+\eta M\\
            \leq& \eta(G+2M)\leq\frac{\dk}{2}.
        \end{split}
    \end{equation*}
    We can bound $f_t(\x_t)-f_t(\x_t')$ in the same way as Theorem \ref{var-expert}, but   we now use Lemmas \ref{cos3} and \ref{cos1} to bound $f_t(\x_t')-f_t(\u_t)$.
    \begin{equation}
        \begin{split}
        &f_t(\x_t')-f_t(\u_t){\leq} -\langle\expmap_{\x_t'}^{-1}(\u_t),\nabla f_t(\x_t')\rangle\\
        =&\frac{1}{\eta}\langle\expmap_{\x_t'}^{-1}(\u_t),\expmap_{\x_t'}^{-1}(\y_{t+1}')-\expmap_{\x_t'}^{-1}(\y_t)\rangle\\
        {\leq}&\frac{1}{2\eta}\left({\zeta}d(\x_t',\y_{t+1}')^2+d(\x_t',\u_t)^2-d(\y_{t+1}',\u_t)^2\right)-\frac{1}{2\eta}\left(\xi d(\x_t',\y_t)^2+d(\x_t',\u_t)^2-d(\y_t,\u_t)^2\right).\\
        \end{split}
    \end{equation}
    Finally, we need to show
\begin{equation}
    2\eta G+2\eta^2{\zeta}L^2d(\x_t',\y_t)-\xi d(\x_t',\y_t)\leq 0
\end{equation}
Following the proof of Theorem \ref{var-expert}, we find
$$\eta\leq\frac{\xi\delta M}{G+(G^2+2\zeta\xi\delta^2M^2L^2)^{\frac{1}{2}}}$$

satisfies the required condition. The guarantee is thus established.
\end{proof}

\subsection{Proof of Lemma \ref{optHedge}}
\label{app:lem:opthedge}
We first show
\begin{equation}\label{opthedge-eq1}
    \begin{split}
        \sumT\la\l_t,\w_t-\w^*\ra\leq \frac{\ln N+R(\w^*)}{\beta}+\beta\sumT\|\l_t-\m_t\|_{\infty}^2-\frac{1}{2\beta}\sumT\lt(\|\w_t-\w_t'\|_1^2+\|\w_t-\w_{t-1}'\|_1^2\rt)
    \end{split}
\end{equation}
holds for any $\w^*\in\Delta_N$, where
$$
\w_t=\argmin_{\w\in\Delta_N}\;\beta\la\sum_{s=1}^{t-1}\l_s+\m_t,\w\ra +R(\w), \quad t\geq 1
$$
and 
$$
\w_t'=\argmin_{\w\in\Delta_N}\;\beta\la\sum_{s=1}^{t}\l_s,\w\ra +R(\w), \quad t\geq 0.
$$
Note that $R(\w)=\sum_{i\in[N]}w_i\ln w_i$ is the negative entropy. According to the equivalence between Hedge and follow the regularized  leader with the negative entropy regularizer, we have
$
w_{t,i}\propto e^{-\beta\lt(\sum_{s=1}^{t-1}\ell_{s,i}+m_{t,i}\rt)}
$
and 
$
w_{t,i}'\propto e^{-\beta\lt(\sum_{s=1}^{t}\ell_{s,i}\rt)}.
$
To prove Equation \eqref{opthedge-eq1}, we  consider the following
 decomposition:
$$
\la\l_t,\w_t-\w^*\ra=\la \l_t-\m_t,\w_t-\w_t'\ra+\la\m_t,\w_t-\w_t'\ra+\la\l_t,\w_t'-\w^*\ra.
$$
Since $R(\cdot)$ is $1$-strongly convex w.r.t. the $\l_1$ norm, by Lemma \ref{dual}, we have $\|\w_t-\w_t'\|_1\leq\beta\|\l_t-\m_t\|_{\infty}$ and
$$
\la\l_t-\m_t,\w_t-\w_t'\ra\leq\|\l_t-\m_t\|_{\infty}\|\w_t-\w_t'\|_1\leq\beta\|\l_t-\m_t\|_{\infty}^2
$$
by H\"{o}lder's inequality.
Hence it suffices to show
\begin{equation}\label{opthedge-eq2}
    \sumT\la\m_t,\w_t-\w_t'\ra+\la\l_t,\w_t'-\w^*\ra\leq\frac{\ln N+R(\w^*)}{\beta}-\frac{1}{2\beta}\sumT\lt(\|\w_t-\w_t'\|_1^2+\|\w_t-\w_{t-1}'\|_1^2\rt)
\end{equation}
to prove Equation \eqref{opthedge-eq1}.

 Equation \eqref{opthedge-eq2} holds for $T=0$ because $R(\w^*)\geq -\ln N$ holds for any $\w^*\in\Delta_N$. To proceed, we  need the following proposition:
\begin{equation}\label{strong-cvx}
   g(\w^*)+\frac{c}{2}\|\w-\w^*\|^2\leq g(\w) 
\end{equation}
holds for any $c$-strongly convex $g(\cdot):\W\rightarrow\R$ where $\w^*=\argmin_{\w\in\W}g(\w)$. This fact can be easily seen by combining the strong convexity and the first-order optimality condition for convex functions.

Assume Equation \eqref{opthedge-eq2} holds for round $T-1\; (T\geq 1)$ and we denote
$$
C_T=\frac{1}{2\beta}\sumT\lt(\|\w_t-\w_t'\|_1^2+\|\w_t-\w_{t-1}'\|_1^2\rt).
$$
Now for round $T$, 
\begin{equation}
    \begin{split}
        &\sumT\lt(\la \m_t,\w_t-\w_t'\ra+\la\l_t,\w_t'\ra\rt)\\
        \stackrel{\rm (1)}{\leq}&\sum_{t=1}^{T-1}\la\l_t,\w_{T-1}'\ra+\frac{\ln N+R(\w_{T-1}')}{\beta}-C_{T-1}+\la\m_T,\w_T-\w_T'\ra+\la\l_T,\w_T'\ra\\
        \stackrel{\rm (2)}{\leq}&\sum_{t=1}^{T-1}\la\l_t,\w_{T}\ra+\frac{\ln N+R(\w_{T})}{\beta}-C_{T-1}+\la\m_T,\w_T-\w_T'\ra+\la\l_T,\w_T'\ra-\frac{1}{2\beta}\|\w_T-\w_{T-1}'\|_1^2\\
        =&\lt(\sum_{t=1}^{T-1}\la\l_t,\w_{T}\ra+\la\m_T,\w_T\ra+\frac{\ln N+R(\w_{T})}{\beta}\rt)+\la\l_T-\m_T,\w_T'\ra-C_{T-1}-\frac{1}{2\beta}\|\w_T-\w_{T-1}'\|_1^2\\
        \stackrel{\rm (3)}{\leq}&\lt(\sum_{t=1}^{T-1}\la\l_t,\w_{T}'\ra+\la\m_T,\w_T'\ra+\frac{\ln N+R(\w_{T}')}{\beta}\rt)+\la\l_T-\m_T,\w_T'\ra\\
        &-C_{T-1}-\frac{1}{2\beta}\|\w_T-\w_{T-1}'\|_1^2-\frac{1}{2\beta}\|\w_T-\w_{T}'\|_1^2\\
        =&\sumT\la\l_t,\w_T'\ra+\frac{\ln N+R(\w_T')}{\beta}-C_T\\
        \stackrel{\rm (4)}{\leq}&\sumT\la\l_t,\w^*\ra+\frac{\ln N+R(\w^*)}{\beta}-C_T.
    \end{split}
\end{equation}
The first inequality is due to the induction hypothesis with $\w^\star=\w_{T-1}'$. The second and the third ones are applications of Equation \eqref{strong-cvx}. Specifically, $\w_{T-1}'$ and $\w_T$ minimize $\sum_{t=1}^{T-1}\beta\la\l_t,\w\ra+{R(\w)}$ and $\sum_{t=1}^{T-1}\beta\la\l_t,\w\ra+\beta\la\m_T,\w\ra+{R(\w)}$ respectively. The forth inequality follows from $\w_T'$ minimizes $\beta\sumT\la\l_t,\w\ra+R(\w)$.

We now demonstrate how to remove the dependence on $\w_t'$:
\begin{equation}\label{opthedge-eq3}
    \begin{split}
    &\frac{1}{2\beta}\sumT\lt(\|\w_t-\w_t'\|_1^2+\|\w_t-\w_{t-1}'\|_1^2\rt)\\
    \geq&\frac{1}{2\beta}\sumT\lt(\|\w_t-\w_t'\|_1^2+\|\w_{t+1}-\w_t'\|_1^2\rt)-\frac{1}{2\beta}\|\w_{T+1}-\w_T'\|_1^2\\
    \geq& \frac{1}{4\beta}\sum_{t=2}^{T}\|\w_t-\w_{t-1}\|_1^2-\frac{2}{\beta},
    \end{split}
\end{equation}
where the last inequality follows from $\|\x+\y\|^2\leq 2(\|\x\|^2+\|\y\|^2)$ holds for any norm. Now the proof is completed by combining Equations \eqref{opthedge-eq1} and \eqref{opthedge-eq3}.

\subsection{Proof of Theorem \ref{opthedge-reg}}
\label{app:thm:opthedge-reg}
We first show $f_t(\x_t)-f_t(\x_{t,i})\leq\la\w_t,\l_t\ra-\ell_{t,i}$ so that  Lemma \ref{optHedge} can be invoked to bound the regret of the meta algorithm. We start from gsc-convexity,
\begin{equation}
    \begin{split}
        &f_t(\x_t)-f_t(\x_{t,i})\leq -\la\nabla f_t(\x_t),\expmap_{\x_t}^{-1}\x_{t,i}\ra \\
       =&\la\nabla f_t(\x_t),\sum_{j=1}^Nw_{t,j}\expmap_{\x_t}^{-1}\x_{t,j}\ra-\la\nabla f_t(\x_t),\expmap_{\x_t}^{-1}\x_{t,i}\ra\\
       =&\sum_{j=1}^Nw_{t,j}\ell_{t,j}-\ell_{t,i}=\la\w_t,\l_t\ra-\ell_{t,i},
    \end{split}
\end{equation}
where the first equality is due to the gradient of $\sum_{i=1}^Nw_{t,i}d(\x,\x_{t,i})^2$ vanishes at $\x_t$. Now we can bound the regret as
\begin{equation}\label{grad-var-eq1}
    \begin{split}
        &\sum_{t=1}^Tf_t(\x_t)-f_t(\x_{t,i})\leq\sum_{t=1}^T\la\w_t,\l_t\ra-\ell_{t,i}\\
        \leq&\frac{2+\ln N}{\beta}+\beta\sum_{t=1}^T\|\l_t-\m_t\|_{\infty}^2-\frac{1}{4\beta}\sum_{t=2}^T\|\w_t-\w_{t-1}\|_1^{2}\\
    \end{split}
\end{equation}
It now suffices to bound $\|\l_t-\m_t\|_{\infty}^2$ in terms of the gradient-variation  $V_T$ and $\|\w_t-\w_{t-1}\|_1^2$. We start with the definition of the infinity norm.
\begin{equation}\label{grad-var-eq2}
    \begin{split}
        &\|\l_t-\m_t\|_{\infty}^2\\
        =&\max_{i\in[N]}(\ell_{t,i}-m_{t,i})^2\\
        =&\max_{i\in[N]}\left(\la\nabla f_t(\x_t),\expmap_{\x_t}^{-1}\x_{t,i}\ra-\la\nabla f_{t-1}(\bar{\x}_t),\expmap_{\bx_t}^{-1}\x_{t,i}\ra\right)^2\\
        =&\max_{i\in[N]}\Bigl(\Bigl. \la\nabla f_t(\x_t),\expmap_{\x_t}^{-1}\x_{t,i}\ra-\la\nabla f_{t-1}(\x_t),\expmap_{\x_t}^{-1}\x_{t,i}\ra+\la\nabla f_{t-1}(\x_t),\expmap_{\x_t}^{-1}\x_{t,i}\ra\\&
        -\la\Ga_{\bx_t}^{\x_t}\nabla f_{t-1}(\bar{\x}_t),\expmap_{\x_t}^{-1}\x_{t,i}\ra+\la\Ga_{\bx_t}^{\x_t}\nabla f_{t-1}(\bar{\x}_t),\expmap_{\x_t}^{-1}\x_{t,i}\ra-\la\Ga_{\bx_t}^{\x_t}\nabla f_{t-1}(\bar{\x}_t),\Ga_{\bx_t}^{\x_t}\expmap_{\bx_t}^{-1}\x_{t,i}\ra\Bigl.\Bigl)^2\\
        \stackrel{\rm (1)}{\leq}&3\max_{i\in[N]}\Bigl(\Bigl. \la\nabla f_t(\x_t)-\nabla f_{t-1}(\x_t),\expmap_{\x_t}^{-1}\x_{t,i}\ra^2+\la\nabla f_{t-1}(\x_t)-\Ga_{\bx_t}^{\x_t}\nabla f_{t-1}(\bar{\x}_t),\expmap_{\x_t}^{-1}\x_{t,i}\ra^2\\
        &+\la\Ga_{\bx_t}^{\x_t}\nabla f_{t-1}(\bar{\x}_t),\expmap_{\x_t}^{-1}\x_{t,i}-\Ga_{\bx_t}^{\x_t}\expmap_{\bx_t}^{-1}\x_{t,i}\ra^2\Bigl. \Bigl)\\
        \stackrel{\rm (2)}{\leq}&3\left(D^2\sup_{\x\in\N}\|\nabla f_t(\x)-\nabla f_{t-1}(\x)\|^2+D^2L^2d(\x_t,\bx_t)^2+G^2\|\expmap_{\x_t}^{-1}(\x_{t,i})-\Ga_{\bx_t}^{\x_t}\expmap_{\bx_t}^{-1}(\x_{t,i})\|^2\right)
    \end{split}
\end{equation}
where  the first inequality relies on fact that $(a+b+c)^2\leq 3(a^2+b^2+c^2)$ holds for any $a,b,c\in\R$, and the second one follows from Assumptions \ref{diam}, \ref{grad}, $L$-gsc-smoothness and H\"{o}lder's inequality.


By  Lemma \ref{smooth}, for a Hadamard manifold with sectional curvature lower bounded by $\kappa$, $h(\x)\coloneqq\frac{1}{2}d(\x,\x_{t,i})^2$ is $\frac{\sqrt{\kappa}D}{\tanh(\sqrt{\kappa}D)}$-smooth (which is exactly ${\zeta}$ as in Definition \ref{def1}) on Hadamard manifolds. Thus
\begin{equation}\label{grad-var-eq3}
    \|\expmap_{\x_t}^{-1}(\x_{t,i})-\Ga_{\bx_t}^{\x_t}\expmap_{\bx_t}^{-1}(\x_{t,i})\|=\|-\nabla h(\x_t)+\Ga_{\bx_t}^{\x_t}\nabla h(\bx_t)\|\leq {\zeta} d(\x_t,\bx_t).
\end{equation}
We  need to bound $d(\x_t,\bx_t)$ in terms of $\|\w_t-\w_{t-1}\|_1$ to make full use of the negative term in Lemma \ref{optHedge}. By Lemma \ref{frechet}
\begin{equation}\label{grad-var-eq4}
    d(\x_t,\bx_t)\leq\sum_{i=1}^Nw_{t,i}\cdot d(\x_{t,i},\x_{t,i})+D\|\w_t-\w_{t-1}\|_1=D\|\w_t-\w_{t-1}\|_1.
\end{equation}
Combining Equations \eqref{grad-var-eq1}, \eqref{grad-var-eq2}, \eqref{grad-var-eq3} and \eqref{grad-var-eq4}, we have
\begin{equation}
    \sum_{t=1}^Tf_t(\x_t)-f_t(\x_{t,i})\leq\frac{2+\ln N}{\beta}+3\beta D^2(V_T+G^2)+\sumsT\left(3\beta (D^4L^2+D^2G^2{\zeta}^2)-\frac{1}{4\beta}\right)\|\w_t-\w_{t-1}\|_1^2,
\end{equation}
where the  $3\beta D^2G^2$ term is due to the calculation of $V_T$ starts from $t=2$, while $\w_0=\w_1$ ensures $\x_1=\bx_1$ and thus $d(\x_1,\bx_1)=0$.

\subsection{Proof of Theorem \ref{var-reg}}
\label{app:thm:var-reg}

The optimal step size, according to Theorem  \ref{var-expert} is
$$
\eta^\star=\min\left\{\frac{\delta}{1+(1+2\zeta\delta^2L^2)^{\frac{1}{2}}}, \sqrt{\frac{D^2+2DP_T}{2{\zeta}V_T}}\right\}.
$$
Based on Assumption \ref{grad}, we know $V_T$ has an upper bound $V_T=\sum_{t=2}^T\sup_{\x\in\N}\|\nabla f_t(\x)-\nabla f_{t-1}(\x)\|^2\leq 4G^2T$. Therefore, $\eta^\star$ can be bounded by
$$
\sqrt{\frac{D^2}{8{\zeta}G^2T}}\leq\eta^\star\leq\frac{\delta}{1+(1+2\zeta\delta^2L^2)^{\frac{1}{2}}}.
$$
According to the construction of $\H$,
$$
\min\H=\sqrt{\frac{D^2}{8{\zeta}G^2T}},\qquad\max\H\geq 2\frac{\delta}{1+(1+2\zeta\delta^2L^2)^{\frac{1}{2}}}.
$$
Therefore, there always exists $k\in[N]$ such that $\eta_k\leq\eta^\star\leq 2\eta_k$.
We can bound the regret of the $k$-th expert as
\begin{equation}\label{var-expert-eq1}
    \begin{split}
        &\sum_{t=1}^Tf_t(\x_{t,k})-\sum_{t=1}^Tf_t(\u_t)\leq\eta_k{\zeta}V_T+\frac{D^2+2DP_T}{2\eta_k}\leq\eta^\star{\zeta}V_T+\frac{D^2+2DP_T}{\eta^\star}\\
        \leq&\zeta V_T \sqrt{\frac{D^2+2DP_T}{2\zeta V_T}}+{(D^2+2DP_T)}\cdot\max\lt\{\sqrt{\frac{2\zeta V_T}{D^2+2DP_T}}, \frac{1+(1+2\zeta\delta^2L^2)^{\frac{1}{2}}}{\delta} \rt\}\\
        =&\frac{3}{2}\sqrt{2(D^2+2DP_T){\zeta}V_T}+(D^2+2DP_T)\frac{1+(1+2\zeta\delta^2L^2)^{\frac{1}{2}}}{\delta}.\\
    \end{split}
\end{equation}
Since the dynamic regret can be decomposed as the sum of the meta-regret and the expert-regret,
$$
\sum_{t=1}^Tf_t(\x_t)-\sum_{t=1}^Tf_t(\u_t)=\sum_{t=1}^Tf_t(\x_t)-\sum_{t=1}^Tf_t(\x_{t,k})+\sum_{t=1}^Tf_t(\x_{t,k})-\sum_{t=1}^Tf_t(\u_t).
$$

Applying Theorem \ref{opthedge-reg} with   $\beta\leq \frac{1}{\sqrt{12(D^4L^2+D^2G^2{\zeta}^2)}}$, we have
$$
\left(3\beta (D^4L^2+D^2G^2{\zeta}^2)-\frac{1}{4\beta}\right)\leq 0
$$
and
$$
\sum_{t=1}^Tf_t(\x_t)-\sum_{t=1}^Tf_t(\x_{t,i})\leq\frac{2+\ln N}{\beta}+3\beta D^2 (V_T+G^2).
$$
We need to consider two cases based on the value of $\beta$.

If $\sqrt{\frac{2+\ln N}{3D^2(V_T+G^2)}}\leq \frac{1}{\sqrt{12(D^4L^2+D^2G^2{\zeta}^2)}}$, then
$$
\sum_{t=1}^Tf_t(\x_t)-\sum_{t=1}^Tf_t(\x_{t,i})\leq 2\sqrt{3D^2(V_T+G^2)(2+\ln N)}.
$$
Otherwise, we have
$$
\sum_{t=1}^Tf_t(\x_t)-\sum_{t=1}^Tf_t(\x_{t,i})\leq 2(2+\ln N)\sqrt{12(D^4L^2+D^2G^2{\zeta}^2)}.
$$
In sum,
\begin{equation}\label{var-meta}
    \begin{split}
        &\sum_{t=1}^Tf_t(\x_t)-\sum_{t=1}^Tf_t(\x_{t,i})\\
        \leq&\max\left\{2\sqrt{3D^2(V_T+G^2)(2+\ln N)},2(2+\ln N)\sqrt{12(D^4L^2+D^2G^2{\zeta}^2)}\right\}.\\
    \end{split}
\end{equation}

Combining Equations \eqref{var-meta} and  \eqref{var-expert-eq1}, we have
\begin{equation}
    \begin{split}
        &\sum_{t=1}^Tf_t(\x_t)-\sum_{t=1}^Tf_t(\u_t)\\
        \leq &\max\left\{2\sqrt{3D^2(V_T+G^2)(2+\ln N)},2(2+\ln N)\sqrt{12(D^4L^2+D^2G^2{\zeta}^2)}\right\}\\
        &+\frac{3}{2}\sqrt{2(D^2+2DP_T){\zeta}V_T}+(D^2+2DP_T)\frac{1+(1+2\zeta\delta^2L^2)^{\frac{1}{2}}}{\delta}\\
        =&O\left(\sqrt{(V_T+{\zeta}^2\ln N)\ln N}\right)+O\left(\sqrt{{\zeta}(V_T+(1+P_T)/\delta^2)(1+P_T)}\right)\\
        =&{O}\left(\sqrt{{\zeta}(V_T+(1+P_T)/\delta^2)(1+P_T)}\right),
    \end{split}
\end{equation}
where we use ${O}(\cdot)$ to hide $O(\log\log T)$ following \cite{luo2015achieving} and \cite{zhao2020dynamic} and $N=O(\log T)$ leads to $\ln N=O(\log\log T)$. 

\section{Omitted Proof for Section \ref{radars}}
\label{app:radars}

\subsection{Proof of Lemma \ref{self-bound}}
\label{app:lem:self-bound}
By $L$-gsc-smoothness, we have
$$
f(\y)\leq f(\x)+\la\nabla f(\x),\expmap_{\x}^{-1}(\y)\ra+\frac{L\cdot d(\x,\y)^2}{2}.
$$
Setting $\y=\expmap_{\x}\left(-\frac{1}{L}\nabla f(\x)\right)$, we have
\begin{equation}
    \begin{split}
        0\leq& f(\y)\leq f(\x)-\frac{1}{L}\|\nabla f(\x)\|^2+\frac{L}{2}\cdot\frac{1}{L^2}\|\nabla f(\x)\|^2\\
        =&f(\x)-\frac{1}{2L}\|\nabla f(\x)\|^2,
    \end{split}
\end{equation}
where we use the non-negativity of $f$. The above inequality is equivalent to
$$
\|\nabla f(\x)\|^2\leq 2L\cdot f(\x),
$$
in which the constant is two times better than that of \cite{srebro2010smoothness}.

\subsection{Proof of Lemma \ref{sml-expert}}
\label{app:lem:sml-expert}
The  proof is similar to the proof of Theorem \ref{ROGD}. Let $\x_{t+1}'=\expmap_{\x_t}\lt(-\eta\nabla f_t(\x_t)\rt)$, then analog to Equation \eqref{rogd-eq1}, we have
\begin{equation}
    \begin{split}
        f_t(\x_t)-f_t(\u_t)\leq&\frac{1}{2\eta}\lt(\|\invexp_{\x_t}\u_t\|^2-\|\invexp_{\x_{t+1}}\u_{t+1}\|^2+2D\|\invexp_{\u_t}\u_{t+1}\| \rt)+\frac{\eta{\zeta}\|\nabla f_t(\x_t)\|^2}{2}\\
        \leq& \frac{1}{2\eta}\lt(\|\invexp_{\x_t}\u_t\|^2-\|\invexp_{\x_{t+1}}\u_{t+1}\|^2+2D\|\invexp_{\u_t}\u_{t+1}\| \rt)+
        \eta{\zeta} Lf_t(\x_t),
    \end{split}
\end{equation}
where for the second inequality we apply Lemma \ref{self-bound}.
WLOG, we can assume $\u_{T+1}=\u_T$ and sum from $t=1$ to $T$:
$$
\sum_{t=1}^T\lt(f_t(\x_t)-f_t(\u_t)\rt)\leq\frac{D^2+2DP_T}{2\eta}+\eta{\zeta}L\sum_{t=1}^Tf_t(\x_t).
$$
After simplifying, we get
\begin{equation}
    \begin{split}
        \sum_{t=1}^T\lt( f_t(\x_t)-f_t(\u_t)\rt)
        \leq&\frac{D^2+2DP_T}{2\eta(1-\eta{\zeta} L)}+\frac{\eta{\zeta}L\sum_{t=1}^Tf_t(\u_t)}{1-\eta{\zeta} L}\\
        =&\frac{D^2+2DP_T}{2\eta(1-\eta{\zeta} L)}+\frac{\eta{\zeta}LF_T}{1-\eta{\zeta} L}\\
        \leq &\frac{D^2+2DP_T}{\eta}+2\eta{\zeta} L F_T\\
        =&O\left(\frac{1+P_T}{\eta}+\eta F_T\right).\\
    \end{split}
\end{equation}
where $\eta\leq\frac{1}{2\zeta L}$ is used to obtain the second inequality.

\subsection{Proof of Lemma \ref{sml-meta}}
\label{app:lem:sml-meta}
We again apply Lemma \ref{optHedge}, with the surrogate loss $\ell_{t,i}=\la\nabla f_t(\x_t),\invexp_{\x_t}\x_{t,i}\ra$ and  the optimism $m_{t,i}=0$ for any $i\in[N]$.
In this way, 
\begin{equation}
    \begin{split}
        &\sumT f_t(\x_t)-\sumT f_t(\x_{t,i})\\
        \leq& -\sumT\la\nabla f_t(\x_t),\invexp_{\x_t}\x_{t,i}\ra=\sumT\la\w_t,\l_t\ra-w_{t,i}\\
        \leq&\frac{2+\ln N}{\beta}+\beta\sum_{t=1}^T\|\l_t\|_{\infty}^2-\frac{1}{4\beta}\sum_{t=2}^T\|\w_t-\w_{t-1}\|_1^2\\
        \leq &\frac{2+\ln N}{\beta}+\beta D^2\sumT\|\nabla f_t(\x_t)\|^2\\
        \leq &\frac{2+\ln N}{\beta}+2\beta D^2L\sumT f_t(\x_t)=\frac{2+\ln N}{\beta}+2\beta D^2L\bar{F}_T,
    \end{split}
\end{equation}
where the second inequality follows from Lemma \ref{optHedge}, the third one follows from Assumption \ref{diam} and H\"{o}lder's inequality, while the last inequality is due to Lemma \ref{self-bound}. By setting $\beta=\sqrt{\frac{2+\ln N}{2LD^2\bar{F}_T}}$, the regret of the meta algorithm is upper bounded by
\begin{equation}\label{sml-meta-eq1}
    \sumT f_t(\x_t)-\sumT f_t(\x_{t,i})\leq\sqrt{8D^2L(2+\ln N)\bar{F}_T}=\sqrt{8D^2L(2+\ln N)\sumT f_t(\x_t)}.
\end{equation}

Although $\bar{F}_T$ is unknown similar to the case of Optimistic Hedge, we can use  techniques like the doubling trick or a time-varying step size $\beta_t=O\lt(\frac{1}{\sqrt{1+\bar{F}_t}}\rt)$ to overcome this hurdle.

The RHS of Equation \eqref{sml-meta-eq1} depends on the cumulative loss of $\x_t$, which remains elusive. Here we  apply an algebraic fact that $x-y\leq\sqrt{ax}$ implies $x-y\leq a+\sqrt{ay}$ holds for any non-negative $x,y$ and $a$. Then 
\begin{equation}
    \sumT f_t(\x_t)-\sumT f_t(\x_{t,i})\leq 8D^2L(2+\ln N)+\sqrt{8D^2L(2+\ln N)\bar{F}_{T,i}}
\end{equation}
where we remind $\bar{F}_{T,i}=\sumT f_t(\x_{t,i})$.
\subsection{Proof of Theorem \ref{sml-reg}}
\label{app:thm:sml-reg}
Recall the regret of the meta algorithm as in Lemma \ref{sml-meta}:
\begin{equation}\label{sml-eq1}
    \sum_{t=1}^Tf_t(\x_t)-\sum_{t=1}^Tf_t(\x_{t,i})\leq 8D^2L(2+\ln N)+\sqrt{8D^2L(2+\ln N)\bar{F}_{T,i}}.
\end{equation}
On the other hand, we know the regret can be decomposed as
\begin{equation}\label{sml-eq2}
    \begin{split}
        \sumT f_t(\x_t)-\sumT f_t(\u_t)=&\sumT f_t(\x_t)-\sumT f_t(\x_{t,i})+\sumT f_t(\x_{t,i})-\sumT f_t(\u_t)\\
        =&\sumT f_t(\x_t)-\sumT f_t(\x_{t,i})+\bar{F}_{T,i}-F_T.
    \end{split}
\end{equation}
We can first show there exists an almost optimal step size and bound the regret of the corresponding expert. That regret immediately provides an upper bound of $\bar{F}_{T,i}$ in terms of $F_T$. This argument eliminates the dependence on $\bar{F}_{T,i}$ and leads to a regret bound solely depending on $F_T$ and $P_T$.

Now we bound the regret of the best expert. Note that due to Lemma \ref{sml-expert}, the optimal step size is $\eta=\min\lt\{\frac{1}{2{\zeta} L},\sqrt{\frac{D^2+2DP_T}{2{\zeta} LF_T}}\rt\}$. According to Assumptions \ref{diam}, \ref{grad}, $F_T\leq GDT$. Thus the optimal step size $
\eta^\star$ is bounded by
$$
\sqrt{\frac{D}{4LGT}}\leq \eta^\star\leq\frac{1}{2{\zeta} L}.
$$
Due to our construction of $\H$, there exists $k\in[N]$ such that $\eta_k\leq\eta^\star\leq 2\eta_k$.

According to Lemma \ref{sml-expert},
\begin{equation}\label{sml-eq3}
    \begin{split}
        &\sumT f_t(\x_{t,k})-\sumT f_t(\u_t)\\
        \leq&\frac{D^2+2DP_T}{2\eta_k(1-\eta_k{\zeta}L)}+\frac{\eta_k{\zeta}LF_T}{1-\eta_k{\zeta}L}\leq\frac{D^2+2DP_T}{\eta_k}+2\eta_k{\zeta}LF_T\\
        \leq&\frac{2(D^2+2DP_T)}{\eta^\star}+2\eta^\star{\zeta}LF_T\\
        \leq&2(D^2+2DP_T)\lt(2{\zeta}L+\sqrt{\frac{2{\zeta}LF_T}{D^2+2DP_T}}\rt)+2{\zeta}LF_T\cdot\sqrt{\frac{D^2+2DP_T}{2{\zeta}LF_T}}\\
        =&4{\zeta}L(D^2+2DP_T)+3\sqrt{2{\zeta}LF_T(D^2+2DP_T)}\\
        \leq &\sqrt{2\lt(16{\zeta}^2L^2(D^2+2DP_T)^2+18{\zeta}LF_T(D^2+2DP_T)\rt)}
    \end{split}
\end{equation}
where we apply $\eta^\star\leq \sqrt{\frac{D^2+2DP_T}{2\zeta LF_T}}$,  $\frac{1}{\eta^\star}\leq 2{\zeta}L+\sqrt{\frac{2\zeta LF_T}{D^2+2DP_T}}$ and $\sqrt{a}+\sqrt{b}\leq\sqrt{2(a+b)}$.

Now as we combine Equations \eqref{sml-eq1}, \eqref{sml-eq2}, \eqref{sml-eq3},
\begin{equation}
    \begin{split}
        &\sumT f_t(\x_t)-\sumT f_t(\u_t)\\
        \leq& 8D^2L(2+\ln N)+\sqrt{8D^2L(2+\ln N)\bar{F}_{T,k}}+\sqrt{2\lt(16{\zeta}^2L^2(D^2+2DP_T)^2+18{\zeta}LF_T(D^2+2DP_T)\rt)}\\
        \leq &8D^2L(2+\ln N)+\sqrt{8D^2L(2+\ln N)\lt(F_{T}+\sqrt{2\lt(16{\zeta}^2L^2(D^2+2DP_T)^2+18{\zeta}LF_T(D^2+2DP_T)\rt)}\rt)}\\
        &+\sqrt{2\lt(16{\zeta}^2L^2(D^2+2DP_T)^2+18{\zeta}LF_T(D^2+2DP_T)\rt)}\\
        =&O(\sqrt{\zeta({\zeta}(1+P_T)+F_T)(P_T+1)}).
    \end{split}
\end{equation}
where we again use $O(\cdot)$ to hide the $\log\log T$ term.
\section{Omitted Proof for Section \ref{radarb}}
\label{app:radarb}
\subsection{Necessity of Best-of-both-worlds Bound}
\label{bbw}
We highlight the necessity of achieving a best-of-both-worlds bound by computing the Fr\'{e}chet mean in the online setting on the $d$-dimensional unit Poincar\'{e} disk. The Poincar\'{e} disk looks like a unit ball in Euclidean space, but its Riemannian metric blows up near the boundary:
$$
\la \u,\v\ra_{\x}=\frac{4\la \u,\v\ra_2}{(1-\|\x\|_2^2)^2}
$$
and has constant sectional curvature $-1$. We use $\o$ to denote the origin of the Poincar\'{e} ball and $\e_i$ to be the $i$-th unit vector in the standard basis. The Poincar\'{e} ball has the following property \citep{lou2020differentiating}:
\begin{equation*}
\begin{split}   
    &d(\x,\y)=\arcosh\lt(1+\frac{2\|\x-\y\|_2^2}{(1-\|\x\|_2^2)(1-\|\y\|_2^2)} \rt)\\
    &\invexp_{\mathbf{0}}\y=\arctanh(\|\y\|_2)\frac{\y}{\|\y\|_2}.
\end{split}
\end{equation*}

Now consider the following loss function
$$
f_t(\x)=\sum_{i=1}^{2d}\frac{d(\x,\x_{t,i})^2}{2d}
$$

where $\x_{t,i}=\frac{t}{2T}\e_i$ for $1\leq i\leq d$ and $\x_{t,i}=-\frac{t}{2T}\e_{i-d}$ for $d+1\leq i\leq 2d$. We choose $\N$ to be the convex hull of  $\pm \frac{1}{2} \e_i$, $i=1,\dots, d$. And the comparator is $\u_t=\argmin_{\u_t\in\N}f_t(\u_t)$ which is indeed the origin $\o$ due to symmetry. Now we can bound $V_T$ by
\begin{equation}
    \begin{split}
        V_T=&\sumsT\sup_{\x\in\N}\|\nabla f_t(\x)-\nabla f_{t-1}(\x)\|^2\\
        =&\frac{1}{4d^2}\sumsT\sup_{\x\in\N}\lt\|\sum_{i=1}^{2d}\lt(\invexp_{\x}\x_{t,i}-\invexp_{\x}\x_{t-1,i}\rt)\rt\|^2\\
        \leq&\frac{1}{4d^2}\sumsT c\cdot\lt(\sum_{i=1}^{2d}d(\x_{t,i},\x_{t-1,i})\rt)^2=\sumsT c\lt(\arcosh\lt(1+ \frac{2(\frac{t}{2T}-\frac{t-1}{2T} )^2}{(1-(\frac{t}{2T})^2)(1-(\frac{t-1}{2T})^2)}\rt)\rt)^2\\
        \leq & \sumsT c\lt(\arcosh\lt(1+\frac{8}{9T^2}\rt)\rt)^2\\
        \leq& \sumsT c\cdot\lt(\frac{8}{9T^2}+\sqrt{\frac{64}{81T^4}+\frac{16}{9T^2}}\rt)^2=\sumsT c\cdot O\lt(\frac{1}{{T^2}}\rt)=O\lt(\frac{1}{T}\rt),\\
    \end{split}
\end{equation}
where the first inequality is due to triangle inequality and Lemma \ref{compare}. We note that $c$ is a constant depending on the diamater of $\N$ and the sectional curvature of $\M$. The second one is due to $t\leq T$, while the third inequality follows from $\arcosh(1+x)\leq x+\sqrt{x^2+2x}$.

Similarly, we can evaluate $F_T$
\begin{equation}
    \begin{split}
        F_T=&\sumT f_t(\u_t)=\sumT f_t(\o)=\sumT\sum_{i=1}^{2d}\frac{d(\o,\x_{t,i})^2}{2d}\\
        =&\sumT\arcosh\lt(\frac{1+\frac{t^2}{4T^2}}{1-\frac{t^2}{4T^2}} \rt)=\int_{0}^T\arcosh \lt(\frac{1+\frac{t^2}{4T^2}}{1-\frac{t^2}{4T^2}} \rt)dt+O(1)\\
        =&2T\int_{0}^{\frac{1}{2}}\arcosh\frac{1+a^2}{1-a^2}da+O(1)\\
        =&2T\lt(a\cdot\arcosh\frac{1+a^2}{1-a^2}+\ln(1-a^2)\rt)\big|_{0}^{\frac{1}{2}}+O(1)=\Theta(T).
    \end{split}
\end{equation}
When the input losses change smoothly, $V_T\ll F_T$ and the gradient-variation bound is much tighter than the small-loss bound.

There also exist scenarios in which the small-loss bound is tighter. We still consider computing the Fr\'{e}chet mean on the Poincar\'{e} disk
$$
f_t(\x)=\sum_{i=1}^nd(\x,\x_{t,i})^2/n,
$$
but assume $\x_{t,i}=\y_i$ when $t$ is odd and $x_{t,i}=-\y_i$ when $t$ is even. We restrict $\y_1,\dots,\y_n\in\mathbb{B}(\frac{\e_1}{2},T^{-\alpha})$ where $\mathbb{B}(\p,r)$ is the geodesic ball centered at $\p$ and with radius $r$. $\N$ is the convex hull of $\mathbb{B}(\frac{\e_1}{2},T^{-\alpha})\cup\mathbb{B}(-\frac{\e_1}{2},T^{-\alpha})$. Since the input sequence is alternating, $\sup_{\x\in \N}\|\nabla f_t(\x)-\nabla f_{t-1}(\x)\|^2$ is a constant over time, and we can lower bounded it by
\begin{equation}
    \begin{split}
        &\sup_{\x\in \N}\|\nabla f_t(\x)-\nabla f_{t-1}(\x)\|^2\\
        =&\sup_{\x\in\N}\lt\|\frac{1}{n}\lt(\sum_{i=1}^n\invexp_{\x}\x_{t,i}-\sum_{i=1}^n\invexp_{\x}\x_{t-1,i}\rt)\rt\|^2\\
        =&\sup_{\x\in\N}\lt\|\frac{2}{n}\sum_{i=1}^n\invexp_{\x}\y_i\rt\|^2\geq \frac{4}{n^2}\lt\|\sum_{i=1}^n\invexp_{\o}\y_i\rt\|^2\\
        =&\frac{4}{n^2}\lt(\lt\| \lt(\sum_{i=1}^n\invexp_{\o}\y_i\rt)^{\parallel}\rt\|^2 + \lt\|\lt(\sum_{i=1}^n\invexp_{\o}\y_i\rt)^{\perp}\rt\|^2\rt)\\
        \geq&\frac{4}{n^2}\lt\| \lt(\sum_{i=1}^n\invexp_{\o}\y_i\rt)^{\parallel}\rt\|^2\geq \frac{4}{n^2}\lt\|n\cdot\arctanh\lt(\frac{1}{2}-T^{-\alpha}\rt)\e_1\rt\|_{\mathbf{0}}^2\\
        =&16\arctanh^2\lt(\frac{1}{2}-T^{-\alpha}\rt)\\
        \geq& 16\lt(\frac{\frac{1}{2}-T^{-\alpha}}{\frac{3}{2}-T^{-\alpha}}\rt)^2=\Omega(1). \\
    \end{split}
\end{equation}
where we use $\a^{\parallel}$ and $\a^{\perp}$ to denote components parallel and orthogonal to the direction of $\e_1$, respectively. The key observation is the lower bound attains when $\lt(\sum_{i=1}^n\invexp_{\o}\y_i\rt)^{\perp}$ is zero, and each $\y_i$ has the smallest component along $\e_1$, i.e., $\y_i=\lt(\frac{1}{2}-T^{-\alpha}\rt)\e_1$. We also use $\arctanh(x)\geq\frac{x}{1+x}$. Thus we have $V_T=\Omega(T)$. Now we consider $F_T$. By Lemma \ref{jensen}, we know that $\u_t$ lies within the same geodesic ball as $\x_{t,i}$, $i\in[n]$. Thus
$$
F_T=\sumT f_t(\u_t)=\sumT\sum_{i=1}^n d(\u_t,\x_{t,i})^2/n\leq T\cdot \lt(\frac{2}{T^\alpha}\rt)^2=O(T^{1-2\alpha}).
$$
We can see whenever $\alpha>0$, $F_T=o(T)$ but $V_T=\Omega(T)$.
\subsection{Proof of Theorem \ref{best-meta}}
\label{app:thm:best-meta}
 By Lemma \ref{optHedge} we have
 \begin{equation}\label{best-eq1}
        \sum_{t=1}^Tf_t(\x_t)-f_t(\x_{t,i})
        \leq\frac{2+\ln N}{\beta}+\beta\sum_{t=1}^T\|\l_t-\m_t\|_{\infty}^2-\frac{1}{4\beta}\sum_{t=2}^T\|\w_t-\w_{t-1}\|_1^{2}
\end{equation}
We  bound $\sum_{t=1}^T\|\l_t-\m_t\|_{\infty}^2$ in terms of $\sum_{t=1}^T\|\l_t-\m_t^v\|_{\infty}^2$ and $\sum_{t=1}^T\|\l_t-\m_t^s\|_{\infty}^2$ as follows. By Assumptions \ref{diam} and \ref{grad}, $\|\l_t-\m_t\|_2^2\leq 4NG^2D^2$ and we can compute  $d_t(\m)= \|\l_t-\m\|_{2}^2$ is $\frac{1}{8NG^2D^2}$-exp-concave.
. We have
\begin{equation}\label{best-eq2}
    \begin{split}
        &\sumT\|\l_t-\m_t\|_{\infty}^2\leq\sumT\|\l_t-\m_t\|_2^2\\
        \leq&\min\lt\{\sumT\|\l_t-\m_t^v\|_2^2,\sumT\|\l_t-\m_t^s\|_2^2\rt\}+8NG^2D^2\ln 2\\
        \leq &N\min\lt\{\sumT\|\l_t-\m_t^v\|_{\infty}^2,\sumT\|\l_t-\m_t^s\|_{\infty}^2\rt\}+8NG^2D^2\ln 2,\\
    \end{split}
\end{equation}
where for the second inequality we use Lemma \ref{hedge} and for the first and the third one the norm inequality $\|\x\|_{\infty}\leq\|\x\|_2\leq\sqrt{N}\|\x\|_{\infty}$ is used.

Combining Equations \eqref{var-meta}, \eqref{best-eq1}, \eqref{best-eq2} and Lemma \ref{sml-meta}, we have
$$
\sumT f_t(\x_t)-\sumT f_t(\x_{t,i})\leq\frac{2+\ln T}{\beta}+\beta N\lt(D^2\min\{3V_T,\bar{F}_T\}+8G^2D^2\ln 2\rt)
$$

holds for any $\beta\leq \frac{1}{\sqrt{12(D^4L^2+D^2G^2{\zeta}^2)}}$ and $i\in[N]$.

Suppose $\beta^\star=\sqrt{\frac{2+\ln N}{N(D^2\min\{3(V_T+G^2),\bar{F}_T\}+8G^2D^2\ln 2)} }\leq \frac{1}{\sqrt{12(D^4L^2+D^2G^2{\zeta}^2)}}$, then
$$
\sumT f_t(\x_t)-\sumT f_t(\x_{t,i})\leq 2\sqrt{(2+\ln N)N(D^2\min\{3(V_T+G^2),\bar{F}_{T}\}+8G^2D^2\ln 2)}.
$$
Otherwise
$$
\sumT f_t(\x_t)-\sumT f_t(\x_{t,i})\leq 2(2+\ln N)\sqrt{12(D^4L^2+D^2G^2{\zeta}^2)}.
$$
In sum, we have
\begin{equation}
    \begin{split}
        &\sumT f_t(\x_t)-\sumT f_t(\x_{t,i})\\
        \leq &\max\lt\{2\sqrt{(2+\ln N)N(D^2\min\{3(V_T+G^2),\bar{F}_{T}\}+8G^2D^2\ln 2)},2(2+\ln N)\sqrt{12(D^4L^2+D^2G^2{\zeta}^2)}\rt\}\\
        =&O(\log T\cdot\min\{V_T,\bar{F}_T\}).
    \end{split}
\end{equation}

 \subsection{Proof of Theorem \ref{best-reg}}
 \label{app:thm:best-reg}
By Theorem \ref{best-meta}, we know 
$$
\sumT f_t(\x_t)-\sumT f_t(\x_{t,i})=O\lt(\sqrt{\ln T(\min\{V_T,\bar{F}_T\})}\rt)
$$
holds for any $i\in [N^v+N^s]$. WLOG, we assume $k$ and $k'$ to be indexes of the best experts for the gradient-variation bound and the small-loss bound, respectively. Then by Theorem \ref{var-reg},
\begin{equation}
    \sumT f_t(\x_{t,k})-f_t(\u_t)\leq\frac{3}{2}\sqrt{2(D^2+2DP_T){\zeta}V_T}+(D^2+2DP_T)\frac{1+(1+2\zeta\delta^2L^2)^{\frac{1}{2}}}{\delta}
\end{equation}
while by Theorem \ref{sml-reg},
\begin{equation}
    \sumT f_t(\x_{t,k'})-f_t(\u_t)\leq \sqrt{2\lt(16{\zeta}^2L^2(D^2+2DP_T)^2+18{\zeta}LF_T(D^2+2DP_T)\rt)}.
\end{equation}
Since the regret admits the following decompositions
\begin{equation}
    \begin{split}
        &\sumT f_t(\x_t)-\sumT f_t(\u_t)\\
        =&\sumT f_t(\x_t)-\sumT f_t(\x_{t,k})+\sumT f_t(\x_{t,k})-\sumT f_t(\u_t)\\
        =&\sumT f_t(\x_t)-\sumT f_t(\x_{t,k'})+\sumT f_t(\x_{t,k'})-\sumT f_t(\u_t)\\,
        \end{split}
        \end{equation}
        we indeed have
\begin{equation}
    \begin{split}
        &\sumT f_t(\x_t)-\sumT f_t(\u_t)\\
        \leq & O\lt(\sqrt{\ln T(\min\{V_T,\bar{F}_T\})}\rt)+\min\lt\{\frac{3}{2}\sqrt{2(D^2+2DP_T){\zeta}V_T}+(D^2+2DP_T)\frac{1+(1+2\zeta\delta^2L^2)^{\frac{1}{2}}}{\delta},\rt .\\
        & \lt.\sqrt{2\lt(16{\zeta}^2L^2(D^2+2DP_T)^2+18{\zeta}LF_T(D^2+2DP_T)\rt)}\rt\}\\
        =&O\lt(\sqrt{\ln T(\min\{V_T,\bar{F}_T\})}\rt)\\
        &+O\lt(\min\lt\{\sqrt{{\zeta}(1+P_T)((1+P_T)/\delta^2+V_T)}, \sqrt{{\zeta}(1+P_T)(\zeta(1+P_T)+F_T)}\rt\}\rt)
    \end{split}
\end{equation}
Note that $\bar{F}_T$ can be processed similarly as in Lemma \ref{sml-meta} and Theorem \ref{sml-reg} to get $F_T$.
In sum, the regret is bounded by
$$
\sumT f_t(\x_t)-\sumT f_t(\u_t)=O\lt(\sqrt{\zeta(P_T(\zeta +1/\delta^2)+\min\{V_T,F_T\}+1)(1+P_T)+\ln T\min\{V_T,F_T\}}\rt).
$$
\section{Technical Lemmas}
\label{app:lem}
We  need the following technical lemmas.
\begin{lemma}\label{gsc-mean}\cite[Theorem 2.1]{bento2021elements}
Suppose $f:\M\rightarrow\R$ is geodesically convex and $\M$ is Hadamard. The  geodesic mean $\bar{\x}_k$ w.r.t coefficients $a_1,\dots,a_N$ ($\sum_{i=1}^Na_i=1$, $a_i\geq 0$) is defined as:
\begin{equation}
    \begin{split}
        &\bar{\x}_1=\x_1\\
        &\bar{\x}_k=\expmap_{\bar{\x}_{k-1}}\left(\frac{a_k}{\sum_{i=1}^ka_i}\expmap_{\bar{\x}_{k-1}}^{-1}\x_k\right),\quad k>1.
    \end{split}
\end{equation}
Then we have
\begin{equation}
    f(\bar{\x}_N)\leq\sum_{i=1}^Na_i f(\x_i).
\end{equation}
\end{lemma}

\begin{proof}

We use induction to show a stronger statement
$$
f(\bar{\x}_k)\leq\sum_{i=1}^k\frac{a_i}{\sum_{j=1}^k a_j} f(\x_i)
$$
holds for $k=1,\dots,N$.

For $k=1$, this is obviously true because $\bar{\x}_1=\x_1$. Suppose 
$$
f(\bar{\x}_k)\leq\sum_{i=1}^k\frac{a_i}{\sum_{j=1}^k a_j} f(\x_i)
$$
holds for some $k$, then by geodesic convexity,
\begin{equation}
    \begin{split}
        f(\bar{\x}_{k+1})\leq& \left(1-\frac{a_{k+1}}{\sum_{j=1}^{k+1}a_j}\right)f(\bar{\x}_k)+\frac{a_{k+1}}{\sum_{j=1}^{k+1}a_j}f(\x_{k+1})\\
        \leq&\sum_{i=1}^k\frac{a_i}{\sum_{j=1}^{k+1}a_j}f(\x_i)+\frac{a_{k+1}}{\sum_{j=1}^{k+1}a_j}f(\x_{k+1})\\
        =&\sum_{i=1}^{k+1}\frac{a_if(\x_i)}{\sum_{j=1}^{k+1}a_j}.
    \end{split}
\end{equation}

The first inequality is due to gsc-convexity: for the geodesic determined by $\gamma(0)=\bx_k$ and $\gamma(1)=\x_{k+1}$ we have $\bx_{k+1}=\gamma\lt(\frac{a_{k+1}}{\sum_{i=1}^{k+1}a_i}\rt)$ and thus $f(\gamma(t))\leq(1-t)f(\gamma(0))+tf(\gamma(1))$. For the second inequality, we use the induction hypothesis.
Given $\sum_{i=1}^Na_i=1$, the lemma is proved.
\end{proof}

The computation of the geodesic averaging is summarized in Algorithm \ref{alg:gsc}, which serves as a sub-routine of \radar.

\begin{algorithm2e}[H]
\caption{Geodesic Averaging}\label{alg:gsc}
\KwData{$N$ points $\x_1,\dots,\x_N\in \N$ and $N$ real weights $w_1,\dots,w_N$.}
 Let $\bar{\x}_{1}=\x_1$\\
\For{$k=2,\dots,N$} {
     $\bar{\x}_k=\expmap_{\bar{\x}_{k-1}}\left(\frac{w_k}{\sum_{i=1}^kw_i}\expmap^{-1}_{\bar{\x}_{k-1}}\x_k\right)$\\
    }
     Return $\bar{\x}_N$.\\
\end{algorithm2e}

\begin{lemma}
\label{frechet}
Suppose  $\x_1,\dots,\x_N,\y_1,\dots,\y_N\in\N$ where $\N$ is a gsc-convex subset of a Hadamard manifold $\M$ and the diameter of $\N$ is upper bounded by $D$. Let $\bar{\x}$ and $\bar{\y}$ be the weighted Fr\'echet mean with respect to coefficient vectors $\a$ and $\b$ $(\a,\b\in\Delta_N)$, defined as $\bar{\x}=\argmin_{\x\in\N}\sum_{i=1}^N a_i\cdot d(\x,\x_i)^2$ and $\bar{\y}=\argmin_{\y\in\N}\sum_{i=1}^N b_i\cdot d(\y,\y_i)^2$.
Then we have
\begin{equation}
    d(\bar{\x},\bar{\y})\leq \sum_{i=1}^N a_i\cdot d(\x_i,\y_i)+D\sum_{i=1}^N|a_i-b_i|.
\end{equation}
\end{lemma}

\begin{proof}
Recall that on Hadamard manifolds, the following inequality   \citep[Prop. 2.4]{sturm2003probability}
$$
d(\x,\y)^2+d(\u,\v)^2\leq d(\x,\v)^2+d(\y,\u)^2+2d(\x,\u)\cdot d(\y,\v)$$
holds for any $\x,\y,\u,\v\in\M$. A direct application of the above inequality yields
\begin{equation}\label{eq2-0}
    d(\x_i,\bar{\y})^2+d(\y_i,\bar{\x})^2\leq d(\x_i,\bar{\x})^2+d(\y_i,\bar{\y})^2+2d(\bar{\x},\bar{\y})\cdot d(\x_i,\y_i)\qquad \forall i\in [N].
\end{equation}
By \cite[Theorem 2.4]{bacak2014computing}:
\begin{equation}
    \label{eq2-1}
    \sum_{i=1}^Na_i\cdot d(\x_i,\bar{\x})^2+\sum_{i=1}^N b_i\cdot d(\y_i,\bar{\y})^2+2d(\bar{\x},\bar{\y})^2\leq \sum_{i=1}^N a_i\cdot d(\x_i,\bar{\y})^2+\sum_{i=1}^N b_i\cdot d(\y_i,\bar{\x})^2
\end{equation}
Multiplying Equation \eqref{eq2-0} by $a_i$, summing from $i=1$ to $n$ and adding Equation \eqref{eq2-1}, we have
\begin{equation}
    \begin{split}
        2d(\bar{\x},\bar{\y})^2\leq& 2d(\bar{\x},\bar{\y})\sum_{i=1}^N a_i\cdot d(\x_i,\y_i)+\sum_{i=1}^N(a_i-b_i) d(\y_i,\bar{\y})^2+\sum_{i=1}^N(b_i-a_i)d(\y_i,\bar{\x})^2\\
        =&2d(\bar{\x},\bar{\y})\sum_{i=1}^N a_i\cdot d(\x_i,\y_i)+\sum_{i=1}^N (a_i-b_i)\cdot (d(\y_i,\bar{\y})-d(\y_i,\bar{\x}))\cdot (d(\y_i,\bar{\y})+d(\y_i,\bar{\x}))\\
        \leq& 2d(\bar{\x},\bar{\y})\sum_{i=1}^N a_i\cdot d(\x_i,\y_i)+2D\sum_{i=1}^N |a_i-b_i|d(\bar{\x},\bar{\y}),
    \end{split}
\end{equation}
where for the last inequality we use the triangle inequality for geodesic metric spaces and Assumption \ref{diam}. Now dividing both sides by $2d(\bar{\x},\bar{\y})$ and we complete the proof.
\end{proof}
\begin{lemma}
\label{cos1}
\cite[Lemma 5]{zhang2016first}. Let $\mathcal{M}$ be a Riemannian manifold with sectional curvature lower bounded by $\kappa \leq 0$. Consider $\N$, a gsc-convex subset of $\M$ with diameter $D$. For a geodesic triangle fully lies within $\N$ with side lengths $a, b, c$, we have
$$
a^{2} \leq {\zeta}(\kappa, D) b^{2}+c^{2}-2 b c \cos A
$$
where ${\zeta}(\kappa, D):=\sqrt{-\kappa} D \operatorname{coth}(\sqrt{-\kappa} D)$.
\end{lemma}
\begin{lemma}
\label{cos2}
\cite[Prop. 4.5]{sakai1996riemannian}
Let $\mathcal{M}$ be a Riemannian manifold with sectional curvature upper bounded by $\kappa\leq 0$. Consider $\N$, a gsc-convex subset of $\M$ with diameter $D$.
For a geodesic triangle fully lies within $\N$ with side lengths $a, b, c$, we have

$$
a^{2} \geq b^{2}+c^{2}-2 b c \cos A.
$$
\end{lemma}

\begin{lemma}\cite[Theorem 2.1.12]{bacak2014convex}\label{proj}
    Let $(\H, d)$ be a Hadamard space and $C\subset \H$ be a closed convex set. Then $\Pi_C{\x}$ is singleton and $d(\x, \Pi_C\x)\leq d(\x,\z)$ for any $\z\in C\setminus\{\Pi_C\x\}$.
\end{lemma}
\begin{lemma}\cite[Prop. 6.1 and Theorem 6.2]{sturm2003probability}\label{jensen}
    Suppose $\x_1,\dots, \x_N\in\N$ where $\N$ is a bounded gsc-convex subset of a Hadamard space. $\bx$ is the weighted Fr\'echet mean of $\x_1,\dots, \x_N\in\N$ w.r.t. non-negative $w_1,\dots,w_N$ such that $\sum_{i=1}w_i=1$ and $f:\N\rightarrow R$ is a gsc-convex function. Then $\bx\in\N$ and
    $$
    f(\bx)\leq \sum_{i=1}^Nw_if(\x_i).
    $$
\end{lemma}
\begin{lemma}\label{lem:tech}
    Let 
    \[
    g(x)\coloneqq \frac{-a+({a^2+bx^2})^{\frac{1}{2}}}{x},
    \]
    where $a,b\in\R^{+}$,
    then $g(x)$ increases on $[0,\infty)$.
\end{lemma}
\begin{proof}
     Taking the derivative w.r.t. $x$, we have
    \[
    \begin{split}
        g'(x)=&\frac{\frac{1}{2}\cdot 2bx(a^2+bx^2)^{-\frac{1}{2}}\cdot x-(-a+(a^2+bx^2)^{\frac{1}{2}})\cdot 1}{x^2}\\
        =&\frac{bx^2-(-a\sqrt{a^2+bx^2}+a^2+bx^2)}{x^2\sqrt{a^2+bx^2}}\\
        =&\frac{a\sqrt{a^2+bx^2}-a^2}{x^2\sqrt{a^2+bx^2}}\geq 0
    \end{split}
    \]
    holds for any $x>0$. By L'{H}\^{o}pital's rule, $g(0)=0$ and $g'(0)=\frac{b}{2a}$. Thus we know $g(x)$ increases on $[0,\infty)$.
\end{proof}
\begin{lemma}\cite[Theorem 3.1]{zhou2019revision}\label{gsc-com}
    On a Hadamard manifold $\M$, a subset $C$
    gsc-convex, iff it contains the geodesic convex combinations of any countable points in $C$.
\end{lemma}
\begin{lemma}\cite[Prop. H.1]{ahn2020nesterov}\label{smooth}
Let $\M$ be a Riemannian manifold with sectional curvatures lower bounded by $-\kappa<0$ and the distance function $d(\x)=\frac{1}{2}d(\x,\p)^2$ where $\p\in\M$. For $D\geq 0$, $d(\cdot)$ is gsc-smooth within the domain $\{\u\in\M:d(\u, \p)\leq D\}$.
\end{lemma}
\begin{lemma}\cite[Corollary 5.6]{ballmann2012lectures}\label{gsc-convex}
    Let $\M$ be a Hadamard space and $C\subset\M$ a  convex subset. Then $d(\z, C)$ is gsc-convex for $\z\in\M$.
\end{lemma}
\begin{lemma}\cite[Section 2.1]{bacak2014convex}\label{levelset}
    Let $\H$ be a Hadamard manifold, $f:\H\rightarrow(-\infty, \infty)$ be a convex lower semicontinuous function. Then any $\beta$-sublevel set of $f$:
    $$
    \{\x\in\H:f(\x)\leq\beta\}
    $$
    is a closed convex set.
\end{lemma}
\begin{lemma}\cite[Lemma 4]{sun2019escaping}\label{compare}
    Let the sectional curvature of $\M$ is in $[-K,K]$ and $\x,\y,\z\in\M$ with pairwise distance upper bounded by $R$. Then
    $$
    \|\invexp_{\x}\y-\invexp_{\x}\z\|\leq(1+c(K)R^2)d(\y,\z).
    $$
\end{lemma}
\begin{lemma}\cite[Lemma 2]{duchi2012dual}\label{dual}
    Let
    $$
    \Pi_{\X}^{\psi}(\z,\alpha)=\argmin_{\x\in\X}\la\z,\x\ra+\frac{\psi(\x)}{\alpha}
    $$
    where $\psi(\cdot)$ is $1$-strongly convex w.r.t. $\|\cdot\|$, then
    $$
    \|\Pi_{\X}^{\psi}(\u,\alpha)-\Pi_{\X}^{\psi}(\v,\alpha)\|\leq\alpha\|\u-\v\|_{\star}.
    $$
\end{lemma}
\begin{lemma}\cite[Prop. 3.1 and Theorem 3.2]{cesa2006prediction}\label{hedge}
Suppose the loss function $\ell_t$ is exp-concave for $\eta>0$, then the regret of Hedge is $\frac{\ln N}{\eta}$, where $N$ is the number of experts.
\end{lemma}

\end{document}

%% file: math_command.tex
\usepackage{paralist,amsmath, amssymb}
\usepackage{algorithm,algorithmic}
\usepackage{mathtools}
 \usepackage{graphicx}
 \usepackage{float}
\usepackage{xcolor}
\usepackage{graphics}

\usepackage{multicol}

 \usepackage{amssymb, multirow, paralist, color}
 \usepackage{textcase, booktabs,makecell}
\usepackage{enumitem}
\def \y {\mathbf{y}}

\def \x {\mathbf{x}}

\def \D {\mathcal{D}}
\def \z {\mathbf{z}}
\def \u {\mathbf{u}}
\def \H {\mathcal{H}}
\def \w {\mathbf{w}}
\def \R {\mathbb{R}}
\def \m {\mathbf{m}}
\def \l {\boldsymbol\ell}

\def \expmap {{\textnormal{Exp}}}
\def \invexp {\textnormal{Exp}^{-1}}
\def \Xtt {\tilde{X}_{t-1}}
\def \Gt {\Gamma_t}
\def \Gtt {\Gamma_{t+1}}
\def \Xt {\tilde{X}}
\def \Gd {\delta G}
\def \ndm {\N_{\delta M}}

\def \F {\mathcal{F}}

\def \W {\mathcal{W}}
\def \N {\mathbb{N}}

\def \v {\mathbf{v}}
\def \M {\mathcal{M}}

\def \p {\mathbf{p}}

\def \a {\mathbf{a}}
\def \b {\mathbf{b}}

\def \M {\mathcal{M}}
\def \N {\mathcal{N}}

\def \C {\mathcal{C}}
\def \rg {\mathsf{grad}\kern 0.1em}
\def \rh {\mathsf{Hess}\kern 0.1em}

\def \o {\mathbf{0}}
\def \V {\mathcal{V}}

\def \m {\mathbf{m}}

\def \F {\mathcal{F}}

\def \V {\mathcal{V}}

\def \e {\mathbf{e}}

\def \X {\mathcal{X}}

\def \S {\mathcal{S}}
\def \dk {\D(\kappa_2)}

\def \Ga {\Gamma}
\def \la {\left\langle}
\def \ra {\right\rangle}
\def \bx {\bar{\x}}
\def \lc {\lceil}
\def \rc {\rceil}
\def \lt {\left}
\def \rt {\right}
\def \sumT {\sum_{t=1}^T}
\def \sumsT {\sum_{t=2}^T}

\newtheorem{ass}{Assumption}
\newtheorem{thm}{Theorem}

\newtheorem{defn}{Definition}
\DeclareMathOperator\arctanh{arctanh}
\DeclareMathOperator\arcosh{arcosh}

\DeclareMathOperator*{\argmin}{argmin}

\newcommand{\RomanNumeralCaps}[1]
    {\MakeUppercase{\romannumeral #1}}
\usepackage{microtype}
\usepackage{graphicx}
\usepackage{booktabs} 